\documentclass{article}

\usepackage[left=1.25in, top=1in, bottom=1in, right=1.25in]{geometry}
\usepackage{times}

\usepackage[dvipsnames]{xcolor}

\usepackage{graphicx} 
\usepackage{subfigure} 

\usepackage{natbib}

\usepackage{algorithm}
\usepackage{algorithmic}

\usepackage[colorlinks=true,citecolor=blue]{hyperref}       

\usepackage{url}            
\usepackage{booktabs}       
\usepackage{amsfonts}       
\usepackage{nicefrac}       
\usepackage{microtype}      

\usepackage{amsmath,amssymb,amsthm}

\usepackage[mathscr]{eucal}
\usepackage{stmaryrd}
\usepackage{bbold}

\usepackage{mdframed}
\usepackage{ifthen}

\newboolean{randClass}
\setboolean{randClass}{true}

\usepackage[shortlabels]{enumitem}

\usepackage[utf8]{inputenc} 
\usepackage[T1]{fontenc}    

\usepackage{float}
\usepackage[export]{adjustbox}

\usepackage{newtxtext,newtxmath}

\usepackage{setspace}
\usepackage{booktabs,colortbl,multirow}

\graphicspath{{figs//}}

\graphicspath{{figs//}}

\DeclareFontFamily{U}{mathx}{\hyphenchar\font45}
\DeclareFontShape{U}{mathx}{m}{n}{<-> mathx10}{}
\DeclareSymbolFont{mathx}{U}{mathx}{m}{n}
\DeclareMathAccent{\widebar}{0}{mathx}{"73}

\newcommand{\eg}{e.g.\ }
\newcommand{\ie}{i.e.\ }

\newcommand{\half}{\nicefrac{1}{2}}

\newcommand{\defEq}{\stackrel{.}{=}}

\newcommand{\indicator}[1]{\llbracket #1 \rrbracket}

\newcommand{\argmin}[2]{\underset{#1}{\operatorname{Argmin }}\, #2}

\renewcommand{\Pr}{\mathbb{P}}
\newcommand{\E}[2]{\underset{#1}{\mathbb{E}}\left[ #2 \right]}

\newcommand{\X}{\mathsf{X}}
\newcommand{\Y}{\mathsf{Y}}
\newcommand{\YHat}{\hat{\mathsf{Y}}}

\newcommand{\FCal}{\mathscr{F}}

\newcommand{\XCal}{\mathscr{X}}

\newcommand{\Real}{\mathbb{R}}

\newcommand{\Labels}{\{ 0, 1 \}}

\newcommand{\scorer}{s \colon \XCal \to \Real}
\newcommand{\classifier}{f \colon \XCal \to \Labels}
\newcommand{\randClass}{f}
\newcommand{\randClassifier}{\randClass \colon \XCal \to [0, 1]}

\newcommand{\D}{D}

\newcommand{\DFull}{D_{\mathrm{jnt}}}

\newcommand{\YSens}{\bar{\Y}}
\newcommand{\etaSens}{\bar{\eta}}
\newcommand{\piSens}{\bar{\pi}}

\newcommand{\etaDP}{\eta_{\mathrm{DP}}}
\newcommand{\etaEOO}{\eta_{\mathrm{EO}}}

\newcommand{\DSens}{\bar{\D}}
\newcommand{\DDP}{\bar{\D}_{\mathrm{DP}}}
\newcommand{\DEOO}{\bar{\D}_{\mathrm{EO}}}

\newcommand{\Perf}{R_{\mathrm{perf}}}
\newcommand{\Fairness}{R_{\mathrm{fair}}}
\newcommand{\FairnessSymm}{\Fairness^\diamond}

\newcommand{\Joint}{R_{\mathrm{full}}}

\newcommand{\PhiClass}{\Phi}
\newcommand{\PhiFair}{{\PhiClass}_{\mathrm{fair}}}
\newcommand{\PhiFairSymm}{\PhiFair^\diamond}

\newcommand{\DISymm}{\DI^\diamond}
\newcommand{\MDSymm}{\MD^\diamond}

\newcommand{\FairnessMeasure}{\Fairness( \cdot; \DFull ) \colon [0, 1]^{\XCal} \to \Real_+}

\newcommand{\FPR}{\mathrm{FPR}}
\newcommand{\FNR}{\mathrm{FNR}}

\newcommand{\BER}{\mathrm{BER}}
\newcommand{\CSERR}{\mathrm{CS}}
\newcommand{\CSERRBal}{\mathrm{CS}_{\mathrm{bal}}}
\newcommand{\DI}{\mathrm{DI}}
\newcommand{\MD}{\mathrm{MD}}

\newcommand{\FPRSens}{\FPR( f; \DSens )}
\newcommand{\FNRSens}{\FNR( f; \DSens )}

\newcommand{\BERSens}{\BER( f; \DSens )}
\newcommand{\DISens}{\DI( f; \DSens )}
\newcommand{\MDSens}{\MD( f; \DSens )}

\newcommand{\regret}{\mathrm{reg}}

\definecolor{dg}{RGB}{2,101,15}
\newtheoremstyle{dotlessG}{}{}{}{}{\color{dg}\bfseries}{.}{ }{}
\definecolor{db}{RGB}{2,15,101}
\newtheoremstyle{dotlessB}{}{}{}{}{\color{db}\bfseries}{.}{ }{}
\definecolor{dr}{RGB}{101,2,101}
\newtheoremstyle{dotlessR}{}{}{}{}{\color{dr}\bfseries}{.}{ }{}

\theoremstyle{plain}
\newtheorem{proposition}{Proposition}
\newtheorem{corollary}[proposition]{Corollary}

\newtheorem{lemma}[proposition]{Lemma}

\theoremstyle{dotlessR}

\theoremstyle{dotlessB}
\newtheorem{defn}{Definition}

\theoremstyle{dotlessB}
\newtheorem{rem}{Remark}

\theoremstyle{dotlessB}
\newtheorem{ex}{Example}

\newmdtheoremenv[innertopmargin=3pt,innerbottommargin=3pt,innerleftmargin=2pt,innerrightmargin=2pt,skipbelow=0pt,backgroundcolor=purple!2]{problem}{Problem}

\newenvironment{example}
{\begin{ex}}{\end{ex}}

\begin{document} 

\title{The cost of fairness in classification}
\author{Aditya Krishna Menon and Robert C. Williamson \\ Data61 and the Australian National University \\ \{\url{aditya.menon}, \url{bob.williamson}\}@\url{data61.csiro.au}}

\maketitle

\begin{abstract}
We study the problem of learning classifiers with a fairness constraint,
with three main contributions towards the goal of quantifying the problem's inherent tradeoffs.
First,
we relate two existing fairness measures to cost-sensitive risks.
Second, 
we show that for cost-sensitive classification and fairness measures,
the optimal classifier is an instance-dependent thresholding of the class-probability function. 
Third,
we show how the tradeoff between accuracy and fairness is determined by the alignment between the class-probabilities for the target and sensitive features.
Underpinning our analysis is a general framework that
casts the problem of learning with a fairness requirement as one of minimising the difference of two statistical risks.
\end{abstract}

\section{Introduction}

Suppose we wish to learn a classifier to determine suitable candidates for a job.
This classifier may accept as inputs various characteristics about a candidate, such as their interview performance, qualifications, and years of experience.
Suppose one of these characteristics is deemed sensitive, \eg their race.
Then, we might be required to constrain the classifier to not be overly discriminative with respect to this sensitive feature.
Subject to this constraint, we would of course like our classifier to be as accurate as possible.
This is known as the \emph{fairness-aware learning} problem, 
and has received considerable attention in the machine learning community of late \citep{Pedreshi:2008,Kamiran:2009,Calders:2010,Dwork:2012,Kamishima:2012,Fukuchi:2013,Zafar:2016,Hardt:2016,Zafar:2017}.
The primary focus has been on formalising what constitutes a perfectly fair classifier,
and how one learns a classifier to approximately achieve such fairness.
There have been several distinct proposals in both regards (see \S\ref{sec:app-fairness}).

In this paper, we are interested in the tradeoffs inherent in the problem of learning with a fairness requirement.
We specifically focus on the impact fairness has on two aspects of our original problem:
the \emph{structure of the optimal solution},
and the \emph{degradation in performance}.
Our three main contributions \textbf{C1}---\textbf{C3} comprise analyses of both issues:
\begin{itemize}[topsep=0pt,itemsep=0pt]
	\item[\textbf{C1:}] we provide a reduction of two popular existing fairness measures to cost-sensitive risks (Lemmas \ref{lemm:di-cost-sensitive}, \ref{lemm:cv-ber}).

	\item[\textbf{C2:}] we show that for such cost-sensitive classification and fairness measures,
the optimal fairness-aware classifier is an \emph{instance-dependent thresholding} of the class-probability function (Propositions \ref{prop:bayes-sens-unaware}, \ref{prop:bayes-sens-unaware-eoo}).

	\item[\textbf{C3:}] we provide a measure of the \emph{alignment} between the class-probabilities for the target and sensitive features, which quantifies the degradation in performance owing to the fairness requirement (Propositions \ref{prop:frontier-cutoff}, \ref{prop:bregman}).
\end{itemize}
A consequence of \textbf{C1} is a simple procedure for learning with a fairness requirement, involving training separate class-probability estimators for the target and sensitive features, and combining them suitably (\S\ref{sec:bayes-implications}).
Underpinning our analysis is
a general framework casting the fairness-aware learning problem as one of minimising the difference of two statistical risks (\S\ref{sec:disc-aware-defn}),
which allows for an abstract, generic 
treatment of the problem.



\section{Background and notation}

We fix notation and review relevant background.
Table \ref{tbl:glossary} summarises some core concepts that we refer to frequently.

\subsection{Standard learning from binary labels}

Let $\XCal \subseteq \Real^d$ be a measurable instance space, \eg characteristics of a candidate for a job.
In standard learning from binary labels, we have samples from a distribution $\D$ over $\XCal \times \Labels$, with $( \X, \Y ) \sim \D$.
Here, $\Y$ is some \emph{target feature} we would like to predict, \eg whether to hire a candidate.
Our goal is to output a measurable 
\ifthenelse{\boolean{randClass}}{
\emph{randomised classifier}
parametrised by $\randClassifier$
}{
\emph{classifier} $\classifier$
}
that distinguishes between positive ($\Y = 1$) and negative ($\Y = 0$) instances.
A randomised classifier predicts any $x \in \XCal$ to be positive with probability $f( x )$;
the quality of any such classifier is assessed by
a \emph{statistical risk} $R( \cdot; \D ) \colon [0, 1]^{\XCal} \to \Real_+$
which,
for some $\PhiClass \colon [0, 1]^3 \to \Real_+$, is \citep{Narasimhan:2014}
$$ R( f; \D ) \defEq \PhiClass( \FNR( f; \D ), \FPR( f; \D ), \Pr( \Y = 1 ) ), $$
for
the \emph{false-negative} and \emph{false-positive rates}
\begin{equation}
\label{eqn:fpr-fnr-rand}
\ifthenelse{\boolean{randClass}}{
(\FNR( f; \D ), \FPR( f; \D )) \defEq \left( \E{\X \mid \Y = 1}{ 1 - f( \X ) }, \E{\X \mid \Y = 0}{ f( \X ) } \right),
}
{
\FNR( f; \D ) &\defEq \Pr_{\X \mid \Y = 1}\left( f( \X ) = 0 \right) \\
\FPR( f; \D ) &\defEq \Pr_{\X \mid \Y = 0}\left( f( \X ) = 1 \right),
}
\end{equation}
which are average class-conditional probabilities of error when classifying $x \in \XCal$ as positive with probability $f( x )$.



%


%

\begin{example}
\label{example:cs}
The \emph{cost-sensitive error} with cost parameter $c \in (0, 1)$
is parametrised by $\PhiClass \colon (u, v, p) \mapsto p \cdot (1 - c) \cdot u + (1 - p) \cdot c \cdot v$.
When $c = \Pr( \Y = 1 )$, this is a scaled version of the \emph{balanced error},
\begin{equation}
	\label{eqn:ber}
	R_{\mathrm{bal}}( f; \D ) = (\FNR( f; \D ) + \FPR( f; \D ))/2.
\end{equation}
\end{example}


A \emph{Bayes-optimal randomised classifier} for a risk is any
$f^* \in {\operatorname{Argmin}} \, {R( f; \D )}$.
For a broad class of $\Phi$, the optimal classifier is a (possibly distribution dependent) thresholding of the class-probability function, $f^*( x ) = \indicator{ \eta( x ) > t^*( \D ) }$ \citep{Narasimhan:2014}, where $\eta( x ) \defEq \Pr( \Y = 1 \mid \X = x )$ and $\indicator{ \cdot }$ denotes the indicator function.
For the cost-sensitive error with parameter $c$, the Bayes-optimal classifier is $f^*( x ) = \indicator{ \eta( x ) > c }$ \citep{Elkan:2001}.
These Bayes-optimal classifiers motivate a plugin estimator,
where one thresholds an empirical estimate of $\eta$ 
\citep{Narasimhan:2014}.

\subsection{Fairness-aware learning}
\label{sec:disc-aware-defn}

In fairness-aware learning, one modifies the standard problem of learning from binary labels in two ways. 
The statistical setup is modified by assuming 
that in addition to the target feature $\Y$,
there is some \emph{sensitive feature} $\YSens$ we would like to treat in some special way, \eg the race of a candidate.
The classifier evaluation is modified by assuming that we 
reward classifiers that are
``fair'' in the treatment of $\YSens$.
To make this goal concrete, the literature has studied notions of \emph{perfect} and \emph{approximate fairness}.
%
(We construct a general formalism for the problem using these in Problem \ref{prb:disc-aware}.)

%

\textbf{Perfect fairness}.
We will focus on two simple notions of perfect fairness,
stated in terms of the random variables $\Y, \YSens$, and 
classifier prediction
$\YHat \mid \X \sim \mathrm{Bernoulli}( f( \X ) )$.
(Note that we assume $\YSens$ to be binary.)
The first is \emph{demographic parity} \citep{Calders:2010}, which requires the predictions to be independent of the sensitive feature:
\begin{equation}
	\label{eqn:no-di}
	\Pr( \YHat = 1 \mid \YSens = 0 ) = \Pr( \YHat = 1 \mid \YSens = 1 ).
\end{equation}	
The second is \emph{equality of opportunity} \citep{Hardt:2016}, which requires the predictions to be independent of the sensitive feature, but only for the positive instances:
$$	\Pr( \YHat = 1 \mid \Y = 1, \YSens = 0 ) = \Pr( \YHat = 1 \mid \Y = 1, \YSens = 1 ). $$
Other notions of perfect fairness include equalised odds \citep{Hardt:2016}, and lack of disparate mistreatment \citep{Zafar:2017}.
{Demographic parity} 
has received the most study;
however, it is known to have deficiencies \citep{Dwork:2012,Hardt:2016,Zafar:2017}.

\begin{table}[!t]
	\centering
	\renewcommand{\arraystretch}{1.25}

	\scalebox{0.8}{
	\begin{tabular}{@{}ll@{}}
		\toprule
		\toprule
		\textbf{Symbol} & \textbf{Meaning} \\
		\toprule
		$\X$	 & Instance \\
		$\Y$	 & Target feature \\
		$\YSens$ & Sensitive feature \\
		$\D$ 	 & Distribution $\Pr(\X, \Y)$ \\
		$\DSens$ & One of $\{ \DDP, \DEOO \}$ \\		
		$\DDP$ & Distribution $\Pr(\X, \YSens)$ \\
		$\DEOO$  & Distribution $\Pr(\X, \YSens \mid \Y = 1)$ \\
		\bottomrule
	\end{tabular}
	}
	\vspace{10pt}
	\scalebox{0.8}{
	\begin{tabular}{@{}ll@{}}
		\toprule
		\toprule
		\textbf{Symbol} & \textbf{Meaning} \\
		\toprule
		$f$	 			& Classifier \\
		$\Perf$	    	& Performance measure \\
		$\Fairness$ 	& Fairness measure \\
		$\FairnessSymm$ & Symmetrised fairness \\
		$\eta( x )$ 	& $\Pr( \Y = 1 \mid \X = x )$ \\
		$\etaDP( x )$ & $\Pr( \YSens = 1 \mid \X = x )$ \\
		$\etaEOO( x )$ & $\Pr( \YSens = 1 \mid \X = x, \Y = 1 )$ \\
		\bottomrule
	\end{tabular}
	}
	\vspace{-0.1in}
	\caption{Glossary of commonly used symbols.}
	\label{tbl:glossary}
	\vspace{-0.1in}
\end{table}

%
\label{sec:app-fairness}

%

\textbf{Approximate fairness}.
We will focus on two
\emph{fairness measures}
that quantify the \emph{degree} of fairness a given classifier possesses.
The first is the \emph{disparate impact} (\emph{DI}) \emph{factor} \citep{Feldman:2015}, 
which is the ratio of the probabilities appearing in the definition of demographic parity:
\begin{equation}
	\label{eqn:di-factor}
	\DI( f ) \defEq \frac{\Pr( \YHat = 1 \mid \YSens = 0)}{\Pr( \YHat = 1 \mid \YSens = 1 )}.
\end{equation}
The second is the \emph{mean difference} (\emph{MD}) \emph{score} \citep{Calders:2010}, 
which replaces the ratio with a difference:
\begin{equation}
	\label{eqn:cv-score}
	\MD( f ) \defEq \Pr( \YHat = 1 \mid \YSens = 1) - \Pr( \YHat = 1 \mid \YSens = 0).
\end{equation}
We refer the reader to \citet{Zliobaite:2015} for a survey of other fairness measures, including variants of the above.

A final remark is that the sensitive feature may or may not be available when one trains the classifier (see \S\ref{sec:sens-available}).
Avoiding the use of the sensitive feature by itself does not 
guard against discrimination \citep{Pedreshi:2008}.

%
\subsection{Existing work on fairness}

Fairness has received considerable study in philosophy and welfare economics \citep{Rawls:1971aa,Sen:2009aa}; 
however, with few exceptions \citep{Bimore:1994aa,Binmore:2005aa}, there is little formal utilitarian literature that grapples with fairness.
See Appendix \ref{sec:philosophy} for a more detailed overview.

In the machine learning community,
\citet{Dwork:2012} proposed an approach to guarantee fairness relying on a metric over instances.
\citet{Zemel:2013,Louizos:2015} proposed approaches to learn feature representations that guarantee fairness.
Both methods depend directly on the specific instances $x \in \XCal$.
In contrast, our approach never touches the instances, but only risks;
this has the substantial advantage of avoiding change under reparametrisation,
and avoids the infinite regress of determining what is ``similar''.

\section{Fairness measures as statistical risks}

We present our general view of fairness measures
as statistical risks
where the sensitive feature is the target.
This lets us analyse fairness measures using tools for studying risks.

\subsection{General fairness measures}

To formalise the notion of a fairness measure, we
first specify our statistical setup for fairness-aware learning.
Let $\DFull$ be a joint distribution over $\XCal \times \Labels \times \Labels$,
with random variables $( \X, \YSens, \Y ) \sim \DFull$.
Here,
$\X$ represents the instance,
$\Y$ the target feature,
and $\YSens$ the sensitive feature.
We will be interested in three induced distributions:
we refer to $\Pr( \X, \Y )$ as $\D$, $\Pr( \X, \YSens )$ as $\DDP$, and 
$\Pr( \X, \YSens \mid \Y = 1 )$  as $\DEOO$.
We use $\DSens$ to refer generically to either $\DDP$ or $\DEOO$.

In \emph{fairness-aware learning},
our goal is to output a 
\ifthenelse{\boolean{randClass}}{
randomised classifier\footnote{Here and elsewhere, this is understood to mean a randomised classifier parametrised by $f$.} $\randClassifier$
}
{
classifier $\classifier$
}
with small statistical risk on $\D$,
so that $\Y$ is well predicted;
we will denote this risk by $\Perf( \cdot; \D )$,
and refer to it as a \emph{performance measure}.
In addition to this goal, 
we also want $f$ to have large \emph{fairness measure} $\FairnessMeasure$.
Formally:

\begin{problem}
\label{prb:disc-aware}
Given a distribution $\DFull$,
performance and 
fairness measures $\Perf$, $\Fairness$,
and tradeoff parameter $\lambda > 0$,
minimise the combined risk
	\begin{equation}
		\label{eqn:fairness-lagrange}
		\Joint( f; \DFull, \lambda ) \defEq {\Perf( f; \D ) - \lambda \cdot \Fairness( f; \DSens )}.
	\end{equation}	
\end{problem}

We will primarily focus on the following tractable special case of the above problem. 
(See \S\ref{sec:constrained} for a slight variant.)

\subsection{Classification-type fairness measures}

In Problem \ref{prb:disc-aware}, $\Fairness$ depends on $\DFull$, which is defined over the triplet $(\X, \YSens, \Y)$.
An interesting sub-class of $\Fairness$ are those that depend only on $\DDP$, which is defined over the tuple $(\X, \YSens)$.
As with $\Perf$,
such $\Fairness$ can be written as a statistical risk on $\DDP$:
in particular,
given some $\PhiFair \colon [ 0, 1 ]^3 \to \Real_+$,
we may define
a \emph{classification-type fairness measure} via
$$ \Fairness( f; \DDP ) \defEq \PhiFair( ( \FPR( f; \DDP ), \FNR( f; \DDP ), \Pr( \YSens = 1 ) ) ). $$
Intuitively, we are testing whether we can predict the sensitive feature $\YSens$ from $\X$.
When it is possible to do so well according to $\Fairness$,
we do not have fairness.

Returning to the two fairness measures of \S\ref{sec:app-fairness},
we
observe that
$\FPR( f; \DDP ) = \Pr( \YHat = 1 \mid \YSens = 0 )$ and
$\FNR( f; \DDP ) = \Pr( \YHat = 0 \mid \YSens = 1 )$;
thus, they are expressible as risks.

\begin{example}
The disparate impact factor may be written
\begin{equation}
	\label{eqn:di-fpr}
	\DI( f; \DDP ) \defEq \frac{\FPR( f; \DDP )}{1 - \FNR( f; \DDP )},
\end{equation}
\ie it uses
$ \PhiFair \colon (u, v, p) \mapsto \frac{v}{1 - u}. $
\end{example}


\begin{example}
The mean difference score may be written
\begin{equation}
	\label{eqn:md-fpr}
	\MD( f; \DDP ) \defEq 1 - \FNR( f; \DDP ) - \FPR( f; \DDP ), 
\end{equation}
\ie it uses
$ \PhiFair \colon (u, v, p) \mapsto 1 - (u + v). $
\end{example}

%

We make a few remarks on this class of fairness measures.
First, by casting fairness measures as statistical risks, Equation \ref{eqn:fairness-lagrange} becomes
the problem of minimising the \emph{difference of two statistical risks}.
This is a departure from standard tradeoffs between two risks, where one considers the \emph{sum} rather than the difference;
fairness measures are unusual as we seek to \emph{maximise} the underlying risk $\Fairness$.

Second, we can in principle plug-in any standard $\PhiFair$ and get a sensible measure of fairness.
However, certain $\PhiFair$ may be more convenient to work with,
\eg 
from the point of view of interpretability;
this is the case for disparate impact, which has roots in the 80\% rule of the U.S.\ Equal Employment Opportunity Commission \citep{EEOC:1979}.


Third, there is no requirement to restrict attention to $\DDP$.
In particular, we could equally use a risk on $\DEOO$, yielding
$$ \Fairness( f; \DEOO ) \defEq \PhiFair( ( \FPR( f; \DEOO ), \FNR( f; \DEOO ), \Pr( \YSens = 1 ) ) ), $$
which aligns with the equality of opportunity objective.

Fourth, in general one needs to impose additional structure on $\Fairness$
to guarantee 
fairness, as we now discuss.

\subsection{Anti-classifiers and symmetrised fairness}
\label{sec:anti-class}

Employing a statistical risk for $\Fairness$
in Equation \ref{eqn:fairness-lagrange}
constrains the false-positive and negative rates.
However, 
these constraints may assume our classifier is non-trivial on $\DDP$;
as an example, if a classifier $f$ has $\MD( f; \DDP ) = \tau$,
then $\MD( 1 - f; \DDP ) = 1 - \tau$.
Thus, one might be able to deceive such measures via an \emph{anti-classifier}; \ie one which has high fairness, but whose negation has low fairness.

Intuitively, one wishes to disallow such a trivial transformation from adversely affecting fairness.
A simple way to do this is to consider the \emph{symmetrised fairness measure}
\begin{equation}
	\label{eqn:fairness-symm}
	\FairnessSymm( f; \DFull ) \defEq \Fairness( f; \DFull ) \land \Fairness( 1 - f; \DFull ),
\end{equation}
where $\land$ denotes the $\min$ operation.
Maximising Equation \ref{eqn:fairness-symm} requires that \emph{both} the classifier and the anti-classifier perform well.
Such symmetrised measures simply modify the underlying $\PhiFair$:
note that $\FPR( 1 - f ) = 1 - \FPR( f )$, and similarly for $\FNR( 1 - f )$.
Thus, $\FairnessSymm$ is parametrised by 
$$ \PhiFairSymm( u, v, p ) \defEq \PhiFair( u, v, p ) \land \PhiFair( 1 - u, 1 - v, p ). $$
In \S\ref{sec:cs-fair}, 
we show that
a broad class of $\PhiFair$
have $\PhiFairSymm$ maximised when $f \equiv \half$;
\ie a \emph{completely random classifier} is \emph{maximally fair}.
(We can equally enforce that $f \equiv \piSens$
is maximally fair
via a simple correction;
see Appendix \ref{sec:asymm-fairness}.)

\subsection{Relation to existing work}

The notion that statistical risks on $\DDP$ are suitable as fairness measures is implicit in prior surveys of such measures. %
Formalising this notion lets us
subsequently
use tools for studying risks to analyse a range of fairness measures.

The need for symmetrised fairness has not received much attention,
with works employing the MD and DI scores \eg
\citet{Calders:2010,Feldman:2015} implicitly assuming that learned classifiers will perform better than random guessing on $\DDP$.

\section{A cost-sensitive view of fairness measures}
\label{sec:cost-sensitive-view}

The previous section cast the DI and MD fairness measures as statistical risks on $\DSens$.
We now show how they may be further related to cost-sensitive risks.
This implies that analysis of cost-sensitive fairness measures suffices to analyse both these measures.
To begin, we first introduce a useful reparameterisation of the standard cost-sensitive risk.


%
\subsection{Balanced cost-sensitive risk}
\label{sec:bal-cs}

Recall that the standard cost-sensitive risk (Example \ref{example:cs}) is parametrised by $\PhiClass \colon (u, v, p)  \mapsto p \cdot (1 - c) \cdot u + (1 - p) \cdot c \cdot v$.
Now define the \emph{balanced cost-sensitive risk} to be
parametrised by
$\PhiClass_{\mathrm{bal}} \colon (u, v, p) \mapsto 2 \cdot \Phi( u, v, \half )$, so that
\begin{equation}
\label{eqn:balanced-cs}	
\begin{aligned}
	\CSERRBal( f; \D, c ) &\defEq (1 - c) \cdot \FNR( f; \D ) + c \cdot \FPR( f; \D ). 
\end{aligned}
\end{equation}
When $c = \half$, we get the balanced error (Equation \ref{eqn:ber}).
In general, this is simply a scaled and reparameterised version of the standard cost-sensitive risk:
we have $\CSERR( f; \D, c ) = (\alpha + \beta) \cdot \CSERRBal( f; \D, c' )$,
where $\alpha = \pi \cdot (1 - c)$, $\beta = (1 - \pi) \cdot c$, and $c ' = {\beta}/{(\alpha + \beta)}$.
This reparameterisation will however prove convenient in analysing existing fairness measures. 

\subsection{Disparate impact and cost-sensitive risk}

Our first result is that the disparate impact factor (Equation \ref{eqn:di-fpr})
can be related to the balanced cost-sensitive risk.
This suggests that study of the latter helps understand the former.

\begin{lemma}
\label{lemm:di-cost-sensitive}
Pick any distribution $\DSens$
and 
\ifthenelse{\boolean{randClass}}{
randomised classifier $\randClass$.
}
{
classifier $\classifier$.
}
Then, for any $\tau \in [0, 1]$,
if $c \defEq \frac{1}{1 + \tau} \in \left[ \frac{1}{2}, 1 \right]$,
\begin{align}	
	\DISens \geq \tau &\iff \CSERRBal( f; \DSens, c ) \geq 1 - c, \label{eqn:di-asymm} \\
	\DISymm( f; \DSens ) \geq \tau &\iff \CSERRBal( f; \D, c ) \in [ 1-c, c ] \label{eqn:di-symm}.
\end{align}
\end{lemma}

We make two remarks.
First, Lemma \ref{lemm:di-cost-sensitive} does not imply that disparate impact \emph{equals} a cost-sensitive risk,
but rather,
that their superlevel sets are related.
This nonetheless means that 
a disparate impact constraint is equivalent to a cost-sensitive constraint,
with the latter being easier to analyse. 

Second, as Lemma \ref{lemm:di-cost-sensitive} holds for \emph{any} distribution $\DSens$,
we can plug in $\DEOO$, yielding an equivalent result for disparate impact in an ``equality of opportunity'' regime, \ie $\DI( f; \DEOO )$.

\subsection{Mean difference score and balanced error}

Our next result is that the mean difference score (Equation \ref{eqn:md-fpr}) has a strong connection to a balanced cost-sensitive risk.

\begin{lemma}
\label{lemm:cv-ber}
Pick any distribution $\DSens$
and 
\ifthenelse{\boolean{randClass}}{
randomised classifier $\randClass$.
}
{
classifier $\classifier$.
}
Then,
for any $\tau \in [0, 1]$,
if $c = \frac{1 + \tau}{2} \in \left[ \frac{1}{2}, 1 \right]$,
\begin{align}
 \MDSens &= 1 - 2 \cdot \CSERRBal\left( f; \DSens, \half \right) \label{eqn:md-ber} \\
 \MDSens \geq \tau &\iff \CSERRBal\left( f; \DSens, \half \right) \geq 1 - c \nonumber \\
 \MDSymm( f; \DSens ) \geq \tau &\iff \CSERRBal\left( f; \DSens, \half \right) \in \left[ 1 - c, c \right] \nonumber.
\end{align}
\end{lemma}

Thus, the MD score is a transformation of the balanced error,
as the latter corresponds to $c = \half$.
Note that Equation \ref{eqn:md-ber}
implies an equivalence of risks,
and not just super-level sets.

Note also that for the MD score, the corresponding balanced cost-sensitive risk has a cost-parameter that does \emph{not} depend on the chosen $\tau$.
This proves beneficial for the purposes of learning with this measure, as we shall see in \S\ref{sec:bayes-implications}.

\subsection{The cost-sensitive fairness problem}
\label{sec:cs-fair}

The above results establish the versatility of cost-sensitive fairness measures. 
In the sequel, we will thus focus on such measures for general cost parameters, relying on Lemmas \ref{lemm:di-cost-sensitive} and \ref{lemm:cv-ber} to relate statements about them to statements about the DI and MD scores.
For symmetry, we will also focus on cost-sensitive risks for the base problem, albeit with a possibly different cost parameter.

The above requires one tweak:
as per \S\ref{sec:anti-class}, it is desirable to work with symmetrised versions of any fairness measure.
For general balanced cost-sensitive risks,
these symmetrised versions have a simple form:
it is an easy calculation that for any $c \in [0, 1]$ and classifier $f$,
$\CSERRBal( 1 - f; \DSens, \bar{c} ) = 1 - \CSERRBal( f; \DSens, \bar{c} )$.
Thus, the symmetrised version is 
\begin{equation}
	\label{eqn:cs-symm}
	\CSERRBal^\diamond( f; \DSens, \bar{c} ) = \CSERRBal( f; \DSens, \bar{c} ) \land (1 - \CSERRBal( f; \DSens, \bar{c} )).
\end{equation}
This risk is maximised when $\CSERRBal( f; \DSens, c ) = \half$.
A sufficient condition for this is $f \equiv \half$,
so that,
in line with our intuition, a {completely random classifier} is {maximally fair}.

Equipped with this, we can formalise the special case of the general Problem \ref{prb:disc-aware} that is the focus of the sequel.

\begin{problem}
\leavevmode\label{prb:cs-disc-aware}%
Given a distribution $\DFull$,
costs $c, \bar{c}$,
and tradeoff parameter $\lambda \in \Real$,
minimise (for
$\tilde{D} \in \{ \DDP, \DEOO \}$),
\begin{equation}
	\label{eqn:cs-fairness-lagrange}	
	\Joint( f; \D, \tilde{D}, c, \bar{c}, \lambda ) \defEq {\CSERR( f; \D, c ) - \lambda \cdot \CSERR( f; \tilde{D}, \bar{c} )}.	
\end{equation}
\end{problem}

We make three comments on Problem \ref{prb:cs-disc-aware}.
First, we use the standard rather than balanced cost-sensitive risk
as it simplifies the analysis in subsequent sections;
recall from \S\ref{sec:bal-cs} that the two are related by a scaling and reparameterisation.

Second, Equation \ref{eqn:cs-fairness-lagrange} employs the standard (non-symmetrised) cost-sensitive measure, but without a positivity constraint on $\lambda$.
This is because, by Equation \ref{eqn:cs-symm},
a constraint on the symmetrised risk imposes upper and lower bounds on the cost-sensitive risk.
Then, $\lambda$ is the difference in the Lagrange multipliers for these two constraints, which need not be positive;
see Appendix \ref{sec:lagrangian} for more discussion.

Third, there is a subtlety in using Problem \ref{prb:cs-disc-aware} as a proxy for the DI.
As noted above, it is only the superlevel sets of the DI that are related to that of a cost-sensitive risk,
and not the DI itself.
This manifests in the cost parameter $\bar{c}$ itself being a user-specified parameter, unlike for the MD score where it is fixed at $\half$.
We will discuss this issue more in \S\ref{sec:bayes-implications}.

\subsection{Relation to existing work}

Lemma \ref{lemm:di-cost-sensitive} is a special case of a broader relationship between fractional performance measures and ``level-finder'' functions \citep[Theorem 1]{Parambath:2014}, \citep[Lemma 7]{Narasimhan:2015}.
\citet{Feldman:2015} related the disparate impact to the balanced error,
but their bound depends on the distribution and classifier,
while ours uses a cost-sensitive risk with constant $\tau$;
see \S\ref{sec:experiments} and Appendix \ref{sec:di-ber}.

\section{Bayes-optimal fairness-aware classifiers}
\label{sec:bayes-opt}

Having formalised the fairness-aware learning problem,
and having further related existing fairness measures to cost-sensitive risks,
we are in a position to study the tradeoffs imposed by the problem.
We begin by asking:
what impact does the fairness requirement have on the Bayes-optimal solutions?
The structure of these solutions provides insight into the problem,
and also suggests a simple practical algorithm.
In the following, we utilise the following quantities:
\begin{equation}
	\label{eqn:eta}
	\begin{alignedat}{2}	
		\eta( x )  	  &\defEq \Pr( \Y = 1 \mid \X = x ) \qquad &&\pi \defEq \Pr( \Y = 1 ) \\
		\etaDP( x )   &\defEq \Pr( \YSens = 1 \mid \X = x ) \qquad &&\piSens \defEq \Pr( \YSens = 1 ) \\
		\etaEOO( x )   &\defEq \Pr( \YSens = 1 \mid \X = x, \Y = 1 ). \qquad &&\ 
	\end{alignedat}	
\end{equation}

\subsection{Bayes-optimal cost-sensitive classifiers}

We will study the Bayes-optimal classifiers 
of Problem \ref{prb:cs-disc-aware},
so that both our fairness and performance measures are cost-sensitive risks.
When working with $\DDP$
(\ie the demographic parity setting),
Equation \ref{eqn:cs-fairness-lagrange} admits an interesting minimiser.

\begin{proposition}
\label{prop:bayes-sens-unaware}
Pick any distribution $\DFull$, costs $c, \bar{c} \in [0, 1]$, and $\lambda \in \Real$.
Then,
\begin{equation}
	\label{eqn:bayes-class-sens-unaware}
	\begin{aligned}
		\argmin{f \in [0,1]^{\XCal}}{ \Joint( f; \D, \DDP, c, \bar{c}, \lambda ) } = \bigl\{ &f^* \mid (\forall x) \, s^*( x ) \neq 0 \implies f^*( x ) = \indicator{ s^*( x ) > 0} \bigl\},
	\end{aligned}
\end{equation}
\begin{equation}
	\label{eqn:bayes-scorer-sens-unaware}
	(\forall x \in \XCal) \, s^*( x ) \defEq \eta( x ) - c - \lambda \cdot ( \etaDP( x ) - \bar{c} ).
\end{equation}
\end{proposition}

Two comments are in order.
First, we observe that the optimal classifier above is in fact \emph{deterministic}, except for those $x$ for which $s^*( x )$ is exactly $0$.
In general, for a given $\lambda$, we expect this to only hold for few or no $x \in \XCal$.
When $s^*( x ) = 0$, however, then any value of $f^*( x )$ will be optimal.

Second,
assuming $s^*( x ) \neq 0$,
when $\lambda = 0$, the optimal $f^*$ is the familiar Bayes-optimal classifier for a cost-sensitive risk, $\indicator{\eta( x ) > c}$.
For $\lambda \neq 0$, however, we have an \emph{instance dependent threshold correction},
which depends on $\etaDP( x )$.
The correction increases the standard threshold of $c$ whenever
$\etaDP( x ) > \bar{c}$;
intuitively, when we are confident in the sensitive feature being active for an instance,
we are more conservative in classifying the instance as positive.

We now consider a fairness measure that reflects the equality of opportunity notion,
and thus works with $\DEOO$ rather than $\DDP$.
This results in a slightly different Bayes-optimal classifier.

\begin{proposition}
\label{prop:bayes-sens-unaware-eoo}
Pick any $\DFull$, costs $c, \bar{c} \in [0, 1]$, and $\lambda \in \Real$.
Then,
\begin{equation}
	\label{eqn:bayes-class-sens-unaware-eoo}
	\begin{aligned}
		\argmin{f \in [0,1]^{\XCal}}{ \Joint( f; \D, \DEOO, c, \bar{c}, \lambda ) } = \bigl\{ &f^* \mid (\forall x) \, s^*( x ) \neq 0 \implies f^*( x ) = \indicator{ s^*( x ) > 0} \bigl\},
	\end{aligned}
\end{equation}
$$ (\forall x \in \XCal) \, s^*( x ) \defEq \left( 1 - \lambda \cdot \pi^{-1} \cdot ( \etaEOO( \X ) - \bar{c} ) \right) \cdot \eta( x ) - c. $$
\end{proposition}

This result is of the same flavour as Proposition \ref{prop:bayes-sens-unaware}, with two important differences.
First, we 
only need to know the probability of the sensitive feature being active for the positive instances.
Second, the form of the threshold correction is no longer additive,
but multiplicative.

\subsection{Special case: using the sensitive feature as input}
\label{sec:sens-available}

The previous section studied a general $\X$, where the sensitive feature was not necessarily provided as input to the classifier.
The form of the optimal classifier simplifies when we allow the sensitive feature as an input.
We have the following analogue of Proposition \ref{prop:bayes-sens-unaware} when working with $\DDP$.

\begin{corollary}
\label{corr:bayes-sens-aware}
Pick any distribution $\DFull$ where $\D$ includes the sensitive feature, costs $c, \bar{c} \in [0, 1]$, and $\lambda \in \Real$.
For $\eta( x, \bar{y} ) = \Pr( \Y = 1 \mid \X = x, \YSens = \bar{y} )$, 
	\begin{align*}
		\argmin{f \in [0,1]^{\XCal \times \Labels}}{ \Joint( f; \D, \DDP, c, \bar{c}, \lambda ) } = \bigl\{ &f^* \mid (\forall x, \bar{y}) \, s^*( x, \bar{y} ) \neq 0 \implies f^*( x, \bar{y} ) = \indicator{ s^*( x, \bar{y} ) > 0} \bigl\},
	\end{align*}
\begin{align*}
 (\forall x \in \XCal) \, s^*( x, 0 ) &\defEq \eta( x, 0 ) - c + \lambda \cdot \bar{c} \\
 (\forall x \in \XCal) \, s^*( x, 1 ) &\defEq \eta( x, 1 ) - c - \lambda \cdot (1 - \bar{c}).
\end{align*}
\end{corollary}

Here, instead of an instance-dependent threshold,
we simply apply different (constant) thresholds to the class-probabilities for each value of the sensitive feature.
This is a simple consequence of Proposition \ref{prop:bayes-sens-unaware}, as we can simply consider one of the features of $\X$ to be perfectly predictive of the sensitive feature, which makes $\etaDP( x, \bar{y} ) \in \{ 0, 1 \}$.

An analogous special case holds for Proposition \ref{prop:bayes-sens-unaware-eoo},
and is deferred to Corollary \ref{corr:bayes-sens-aware-eoo} of the Appendix.

\subsection{A plugin approach to fairness-aware learning}
\label{sec:bayes-implications}

The Bayes-optimal classifiers derived above
rely on thresholding the class-probabilities $\eta$ and $\etaSens$.
Thus, analogously to the Bayes-optimal classifiers for standard statistical risks,
this motivates a simple plugin estimation approach to 
fairness-aware learning problem:
estimate $\eta, \etaSens$ separately, \eg by logistic regression,
and then combine them as per Equations \ref{eqn:bayes-class-sens-unaware}, \ref{eqn:bayes-scorer-sens-unaware} to construct a classifier. 
When the sensitive feature is available, then all that is needed is a single model for $\eta( x, \bar{y} )$, which is thresholded separately for each of the sensitive feature values.

We make three comments on the proposed approach.
First, one must of course tune $\lambda$ to achieve a desirable tradeoff between accuracy and fairness.
This fortunately does \emph{not} require retraining any model,
as we can simply employ the learned $\eta, \etaSens$ and appropriately change how they are thresholded to form a classifier.
One can tune $\lambda$,
so as to reach some desired operating point on the accuracy-fairness curve.

Second, if we find $s^*( x ) = 0$ for some $x \in \XCal$, \emph{any} prediction for that $x$ optimises the objective of Equation \ref{eqn:cs-fairness-lagrange}; however, we may seek to tune this prediction to favour \eg maximal performance on the original problem.

Third, we reiterate that for the disparate impact, the cost parameter $\bar{c}$ must be tuned as well,
but that as per $\lambda$,
this does not require retraining any model.

%


%
\subsection{Relation to existing work}
\label{sec:related-bayes}

Computing the Bayes-optimal classifiers as above is not without precedent:
\citet{Hardt:2016,Davies:2017} considered the same question, but in the case of \emph{exact} fairness measures.
We are not aware of prior work on computing the optimal classifiers for \emph{approximate} fairness measures.
While the results have a similar flavour to the exact fairness case, explicating them is important to understand the full tradeoff between accuracy and fairness (\S\ref{sec:frontier}),
and also suggests a simple algorithm.

\citet{Hardt:2016} proposed to construct a fairness-aware classifier in the equality of opportunity setting by post-processing the results of a classifier trained on the original problem.
They considered a slightly different constrained version of the problem,
where one forces the solution to have \emph{perfect} rather than approximate fairness.
Our Propositions \ref{prop:bayes-sens-unaware} and \ref{prop:bayes-sens-unaware-eoo} provide an explicit form for the correction when approximate fairness is desired,
as well as when the sensitive feature is available or not during training.
Recently, Woodworth et al.\ \citet{Woodworth:2017}
established limits on the post-processing approach of \citet{Hardt:2016}; studying this in our context of approximate fairness measures would be of interest.

\citet{Calders:2010} proposed to modify the output of na\"{i}ve Bayes so as to minimise the MD score.
However, their approaches do have any theoretical guarantees.

Our plugin learning procedure 
merely requires estimating class-probabilities,
which for logistic regression is a convex problem.
This avoids optimisation challenges facing existing approaches. 
For example, one way to approximately solve Equation \ref{eqn:cs-fairness-lagrange} is to pick convex \emph{surrogate losses} $\ell, \bar{\ell} \colon \Labels \times \Real \to \Real_+$,
and find 
\citep{Zafar:2016,Zafar:2017}
$$ s^* \in \argmin{\scorer}{ \CSERR( s; \D, c, \ell ) - \lambda \cdot \CSERR( s; \DSens, \bar{c}, \bar{\ell} ) } $$
for the \emph{surrogate cost-sensitive risk} \citep{Scott:2012},
\begin{equation}
	\label{eqn:surrogate-cs}
	\CSERR( s; \D, c, \ell ) \defEq \E{(\X, \Y) \sim \D}{ C_{\Y} \cdot \ell( \Y, s( \X ) ) }
\end{equation}
for $C_1 = 1 - c, C_0 = c$.
Note however that for nonlinear $\bar{\ell}$, this objective will be non-convex in $s$.
Even if one manages to overcome this challenge, 
guaranteeing large surrogate fairness does \emph{not} imply large fairness of the underlying classifier, as the former is an \emph{upper bound} to the latter.
Similar problems plague related approaches based on regularisation \citep{Kamishima:2012,Fukuchi:2013}.

\section{Quantifying the accuracy-fairness tradeoff}
\label{sec:unhinged-discrimination}
\label{sec:frontier}

We now study the tradeoff between performance on our base problem and fairness,
and show it is quantifiable
by a measure of \emph{alignment} of the target and sensitive variables.

%
\subsection{The fairness frontier}
\label{sec:constrained}

Our definition of the fairness-aware learning problem (Problem \ref{prb:disc-aware}) was in terms of a linear tradeoff between the performance and fairness measures.
To quantify the tradeoff\footnote{We stress that the tradeoff measured here is one inherent to the \emph{problem}, rather than one owing to the \emph{technique} one uses.} imposed by a fairness constraint,
we will study the following explicitly constrained problem:
for $\tau \in [ 0, 1 ]$,
let
\begin{align}
	f^*_\tau &\in \argmin{\ifthenelse{\boolean{randClass}}{\randClassifier}{\classifier}}{ \Perf( f; \D ) \colon \FairnessSymm( f; \DSens ) \geq \tau } \label{eqn:fairness-constrained} \\
	F( \tau ) &= \Perf( f^*_\tau; \D ) - \Perf( f^*_0; \D ). \label{eqn:frontier}
\end{align}
The function $F \colon [ 0, 1 ] \to \Real_+$
represents the \emph{fairness frontier}:
for a given lower bound on (symmetrised) fairness,
it measures the
\emph{best excess risk} over the solution \emph{without} a fairness constraint.
Evidently,
$F(\cdot)$ is non-decreasing
since
the constraints on $\Fairness$ are nested as $\tau$ increases;
i.e., demanding more fairness can never improve performance. 

As per the previous section,
the case of cost-sensitive performance and fairness measures is of interest.
Here, the objectives in Equations \ref{eqn:fairness-lagrange} and \ref{eqn:fairness-constrained} are related by the Lagrangian principle;
see Appendix \ref{sec:lagrangian} for details.
Further, when $\XCal$ is finite, Equation \ref{eqn:fairness-constrained} reduces to a linear program.
(For infinite $\XCal$ we obtain a semi-infinite linear program \citep{Goberna:1998}, whose duality is subtler to analyse.)

\begin{lemma}
\label{lemm:constrained-opt}
For finite $\XCal$,
pick any distribution $\DFull$, costs $c, \bar{c} \in [0, 1]$, and $\tau \in \Real_+$.
Then, the problem 
$$ \min_{\randClassifier} \CSERR( f; \D, c ) \colon \CSERR^\diamond( f; \DSens, \bar{c} ) \geq \tau $$
is expressible as a linear program.
\end{lemma}




Exploiting this linearity of the objective and constraints,
we can further establish that the frontier is a convex curve.
We caution that this result 
requires equivalence of $\Fairness$ and a cost-sensitive risk, and so is not applicable for the DI factor.

\begin{lemma}
\label{lemm:convex-frontier}
Pick any distribution $\DFull$, and cost-sensitive performance and fairness measures. 
Then, the function $F \colon \Real_+ \to \Real_+$ of Equation \ref{eqn:frontier} is convex.
\end{lemma}


While Lemmas \ref{lemm:constrained-opt} and \ref{lemm:convex-frontier} are useful for computing the frontier,
they do not specify how the curve's behaviour as $\tau$ is varied relates to properties of $\DFull$.
We now study this issue.


%
\subsection{The frontier and class-probability alignment}

%

To obtain our first distribution-dependent statement about the curve,
observe that $F( \tau ) = 0$
for any $\tau \in [0, \tau^*]$,
where $\tau^* = \FairnessSymm( f^*_0; \DSens )$.
This simply means that there is no penalty from a fairness constraint when we consider the fairness attained by the Bayes-optimal classifier for the original problem.
Intuitively, when $\D$ and $\DSens$ are \emph{disaligned}, we expect this $\tau^*$ to be large, so that there is no effect from virtually any fairness constraint.

We can formalise this notion of disalignment.
Recall from Equation \ref{eqn:eta} that $\eta, \etaSens$
are the class-probabilities of the target and sensitive features respectively.
With a cost-sensitive risk for $\Perf$, the Bayes-optimal classifier is $f^*( x ) = \indicator{ \eta( x ) > c }$.
The incurred errors for predicting $\YSens$ with $f^*$ are then
\begin{equation}
	\label{eqn:regret-general}
\begin{aligned}	
	\FPR( f^*; \DSens ) &= \frac{1}{1 - \piSens} \cdot \E{\X}{ (1 - \etaSens( \X )) \cdot \indicator{ \eta( \X ) > c } } \\
	\FNR( f^*; \DSens ) &= \frac{1}{\piSens} \cdot \E{\X}{ \etaSens( \X ) \cdot \indicator{ \eta( \X ) < c } }.
\end{aligned}	
\end{equation}
Both these terms measure a form of disalignment of $\eta$ and $\etaSens$,
specifically looking at the concentration of the latter in regions
where the former is above or below the threshold $c$.
If $\Fairness$ is parametrised by $\PhiFair$,
we then find
$$ \tau^* = \PhiFairSymm( \FPRSens, \FNRSens, \piSens ). $$
For generic $\PhiFair$, we can plug in Equation \ref{eqn:regret-general} to get an explicit expression for $\tau^*$.
In the case of cost-sensitive $\PhiFair$,
this expression involves a concrete measure of disalignment.

\begin{proposition}
\label{prop:frontier-cutoff}
Pick any distribution $\DFull$, and cost-sensitive performance and fairness measures
with cost parameters $c, \bar{c}$.
Then, $F( \tau ) = 0$ for any $\tau \in [0, \tau^*]$, where
\begin{align*}
\tau^* &= \min( \Delta_{c, \bar{c}}( \etaSens, \eta ), \Delta_{-c, \bar{c}}( \etaSens, -\eta ) ) \\
\Delta_{c, \bar{c}}( \etaSens, \eta ) &= \E{\X}{ B_{c, \bar{c}}( \etaSens( \X ), \eta( \X ) ) } - \mathbb{I}_\varphi( \bar{P}_1, \bar{P}_0 ),
\end{align*}
where
$\bar{P}_y = \Pr( \X \mid \YSens = y )$,
$\mathbb{I}_\varphi( \cdot, \cdot )$ denotes an $f$-divergence,
\begin{align}
	\varphi( t ) &\defEq -((1 - \bar{c}) \cdot \piSens \cdot t) \land (\bar{c} \cdot (1 - \piSens)) \nonumber \\
	\label{eqn:bregman-v1}
	B_{c, \bar{c}}( \etaSens, \eta ) &\defEq | \etaSens - \bar{c} | \cdot \indicator{ (\etaSens - \bar{c}) \cdot (\eta - c) < 0 }.	
\end{align}
\end{proposition}

Unpacking the above, Equation \ref{eqn:bregman-v1} gives a concrete notion of disalignment between $\eta$ and $\etaSens$
-- which measures how much they disagree around the respective thresholds $c$ and $\bar{c}$
-- and shows that when this disalignment is high, the fairness constraint has less of an effect.
The additional $f$-divergence term 
is an intrinsic statement about $\DSens$:
when the class-conditionals of $\DSens$ strongly overlap,
\ie there is limited predictability of the sensitive label from the features,
then 
also the fairness constraint has less of an effect.
Finally, the $\min(\cdot,\cdot)$ term arises from using symmetrised fairness.


%

Proposition \ref{prop:frontier-cutoff} specifies how much fairness we can ask for
without paying any performance penalty.
When there is a penalty, however, how does this depend on $\DFull$?
To understand this,
we appeal to Bayes-optimal classifiers to Problem \ref{prb:cs-disc-aware},
whose closed form reveals that the frontier is determined by a similar notion of {disalignment}. 


\begin{proposition}
\label{prop:bregman}
Pick any $\DFull$, and cost-sensitive performance and fairness measures 
with parameters $c, \bar{c}$.
Given $\tau \in [ 0, 1 ]$,
there is some $\lambda \in \Real$
and Bayes-optimal randomised classifier $f^* \in \argmin{f \in [0,1]^{\XCal}}{ R_{\mathrm{full}}( f; \D, \DSens, c, \bar{c}, \lambda ) }$
with
$$ F( \tau ) = \E{\X}{ ( c - \eta( \X )) \cdot ( f^*( \X ) - \indicator{ \eta( \X ) > c } ) }. $$
If further
this $f^*$ is deterministic \ie $\mathrm{Im}( f^* ) \subseteq \{ 0, 1 \}$,
\begin{equation}
	\label{eqn:bregman}
	F( \tau ) = \E{\X}{ B_{\lambda, c, \bar{c}}( \eta( \X ), \etaSens( \X ) ) }
\end{equation}
where $B_{\lambda, c, \bar{c}}( \cdot, \cdot )$ is defined by
$$ B_{\lambda, c, \bar{c}}( \eta, \etaSens ) \defEq | \eta - c | \cdot \indicator{ (\eta - c) \cdot (\eta - c - \lambda \cdot (\etaSens - \bar{c})) < 0 }. $$
\end{proposition}


\begin{figure*}[!t]
	\centering
	\subfigure[$\eta, \etaSens$ given by indicator functions]{
	\begin{minipage}{0.24\textwidth}
	\label{fig:frontier-indicator}
	\includegraphics[scale=0.05]{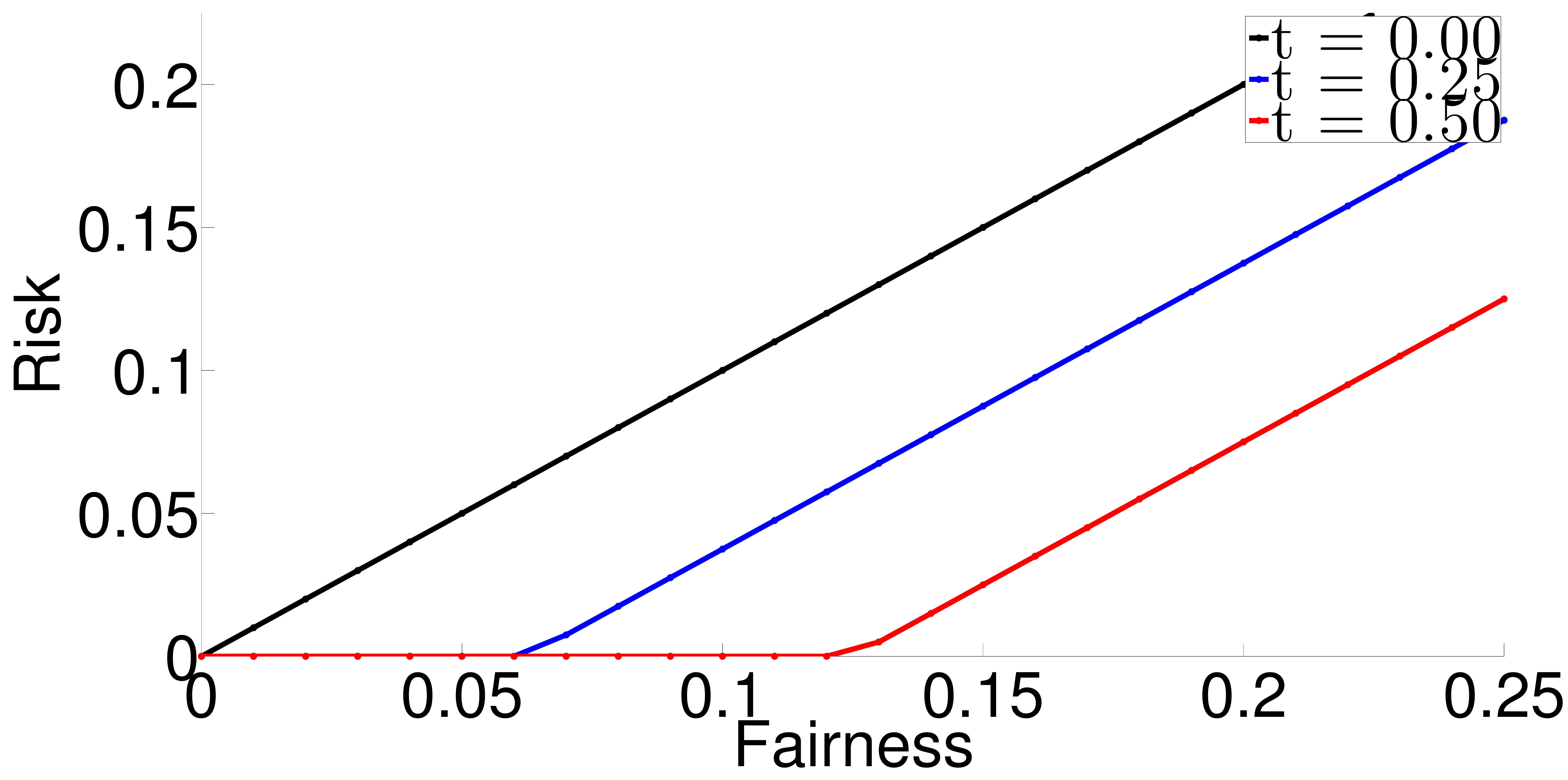}
	\end{minipage}%
	\begin{minipage}{0.24\textwidth}
	\includegraphics[scale=0.05]{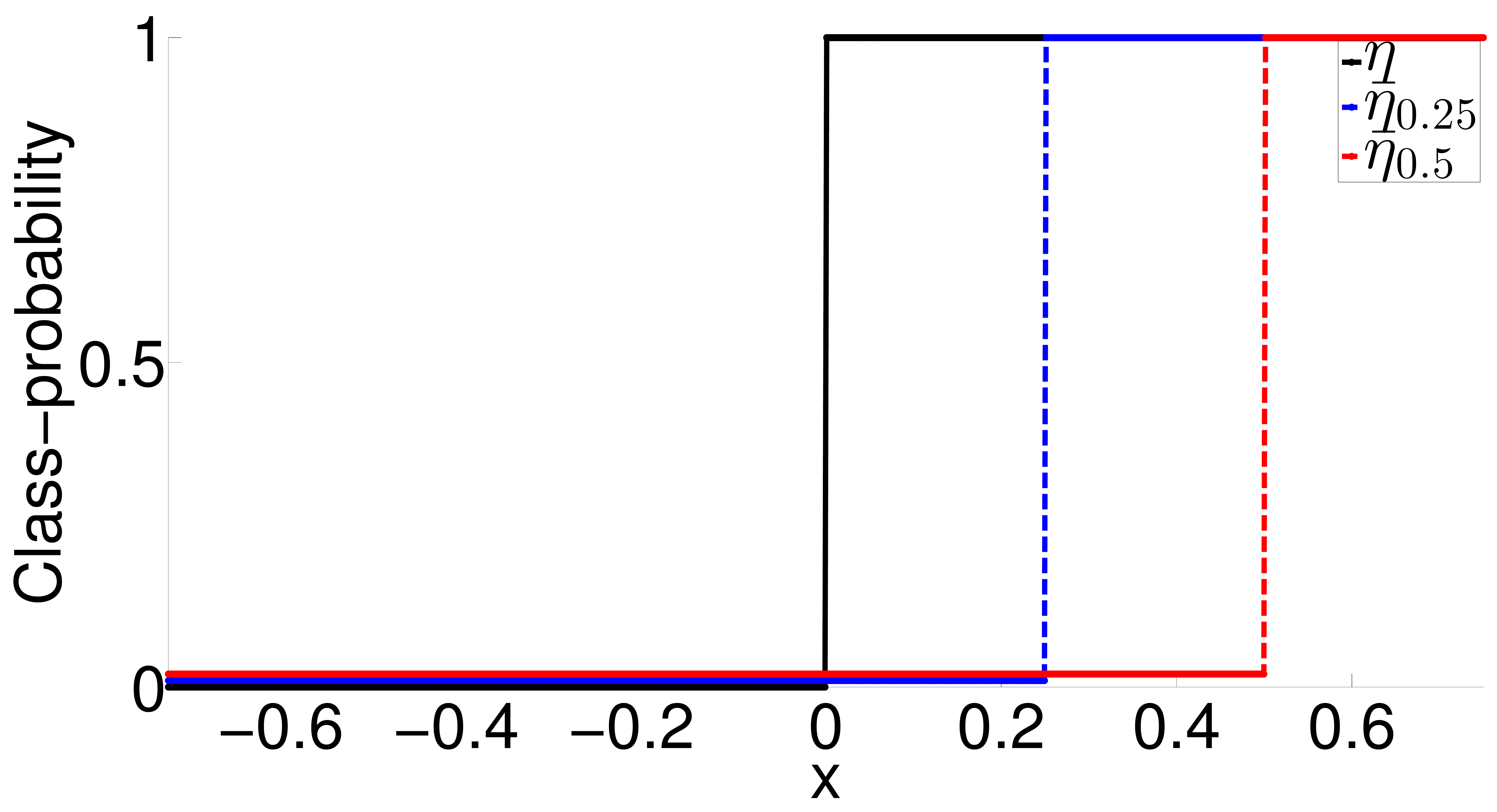}
	\end{minipage}%
	}
	\subfigure[$\eta$ given by sigmoid function, $\etaSens$ by indicator function.]{
	\begin{minipage}{0.24\textwidth}
	\label{fig:frontier-sigmoid} \includegraphics[scale=0.05]{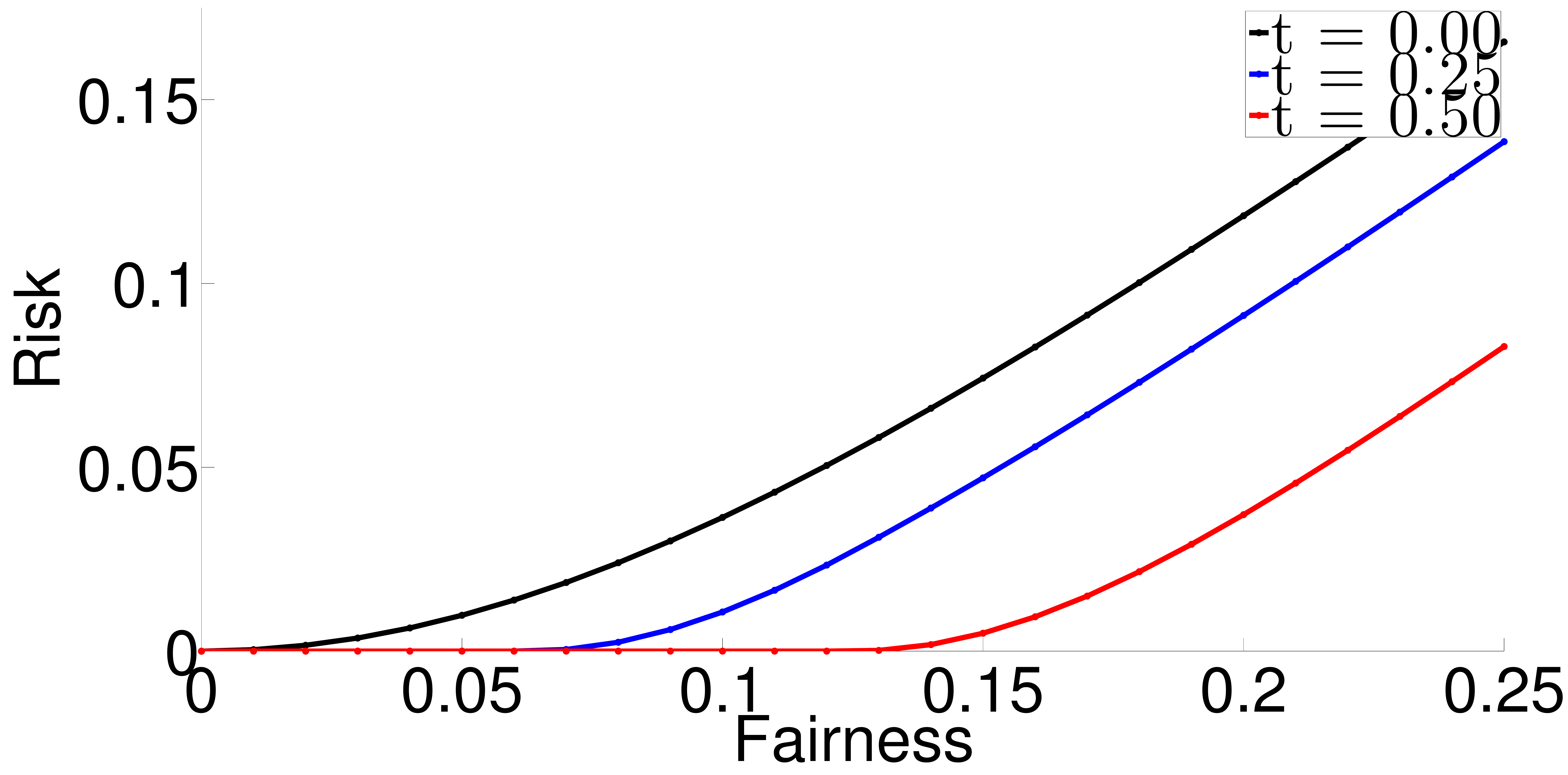}
	\end{minipage}%
	\begin{minipage}{0.24\textwidth}
	\includegraphics[scale=0.05]{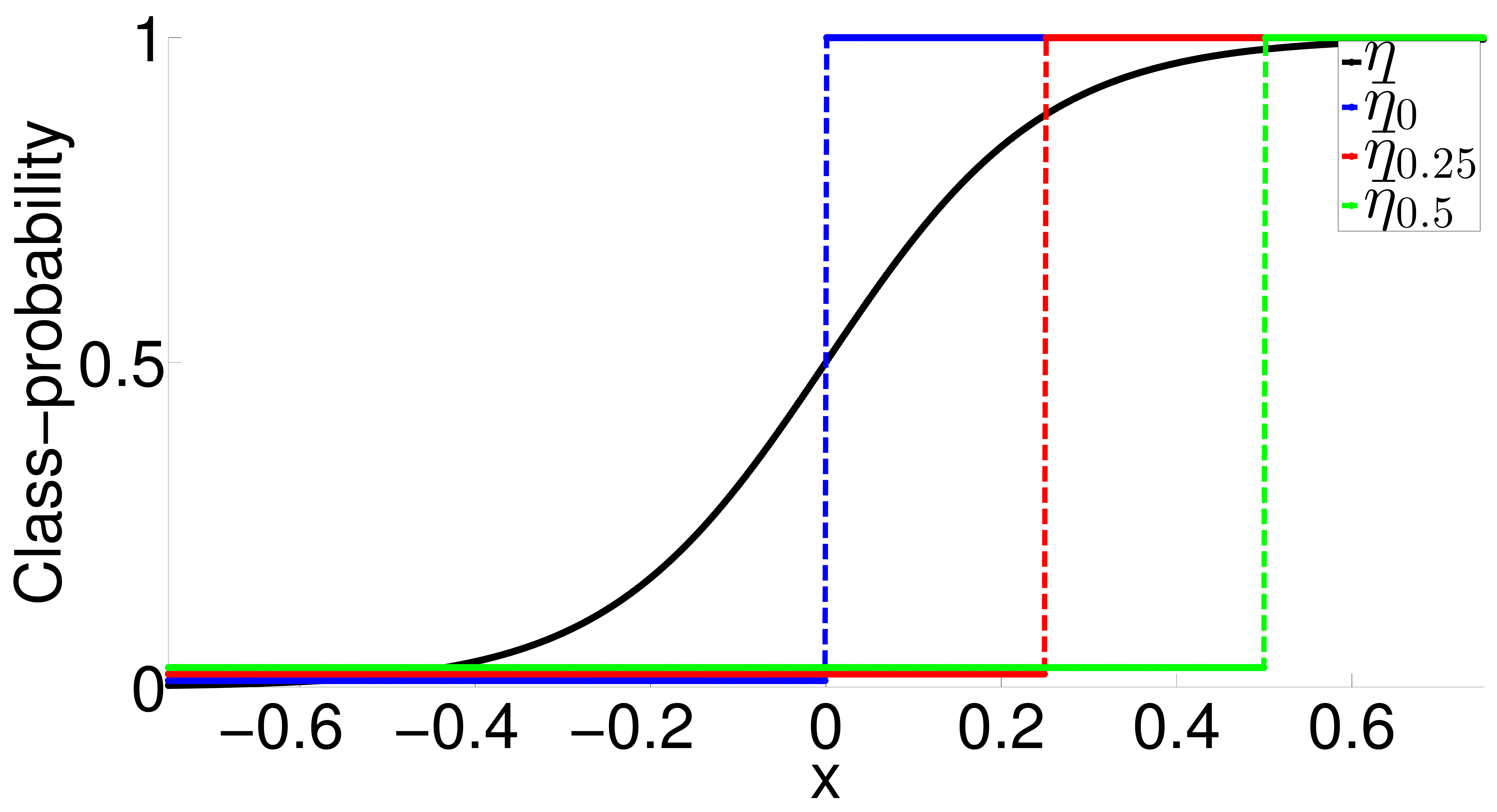}
	\end{minipage}
	}

	\caption{Illustration of fairness frontiers and probability disalignment. See text for description of parameter $t$.}
\end{figure*}

\begin{figure*}[!t]
	\begin{minipage}[valign=top]{0.49\textwidth}
	\centering
	{\includegraphics[scale=0.06]{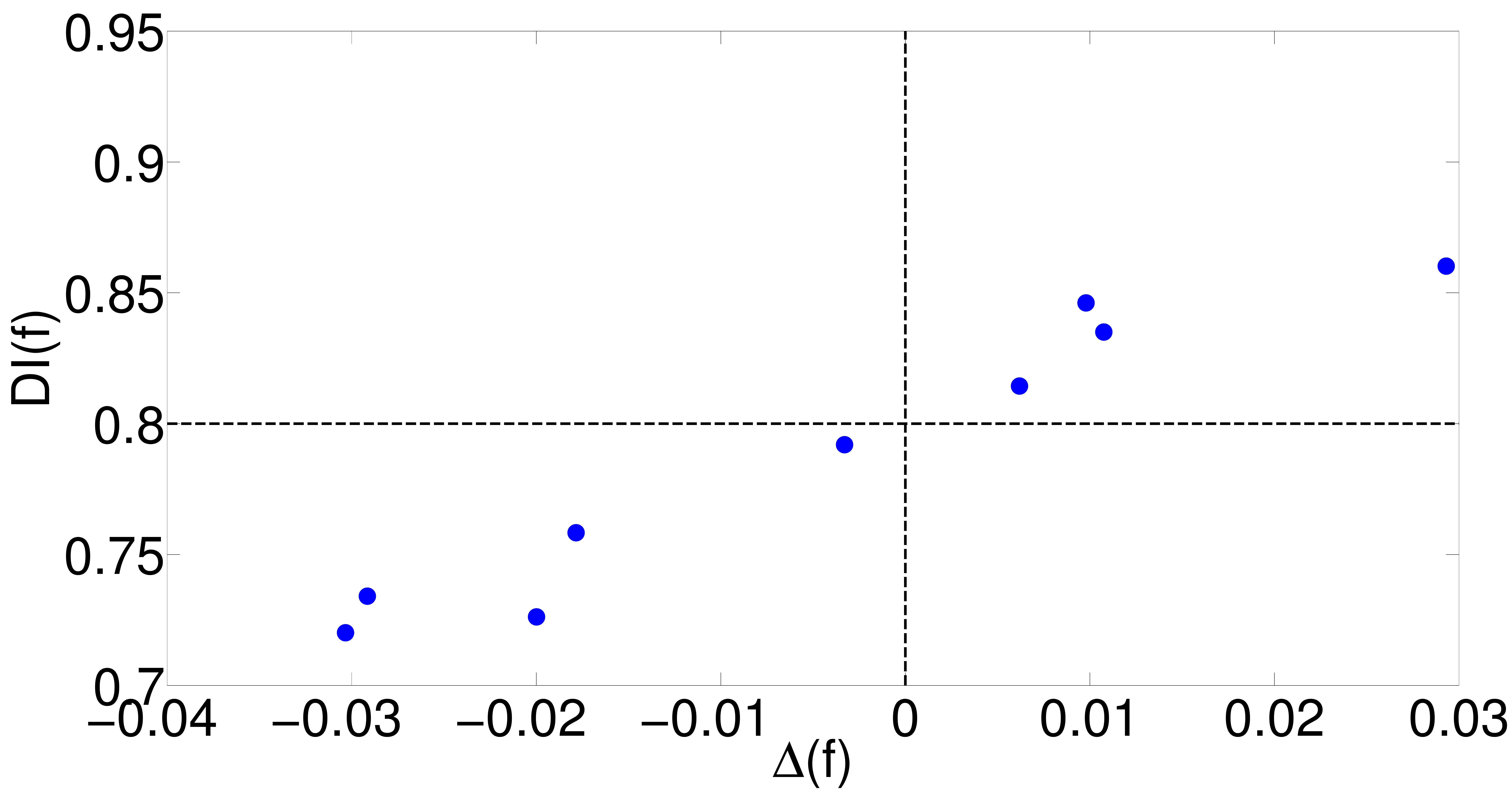} \label{fig:certification}}
	\end{minipage}%
	\vspace{4pt}
	\begin{minipage}[valign=top]{0.49\textwidth}
	\centering
	{\includegraphics[scale=0.1]{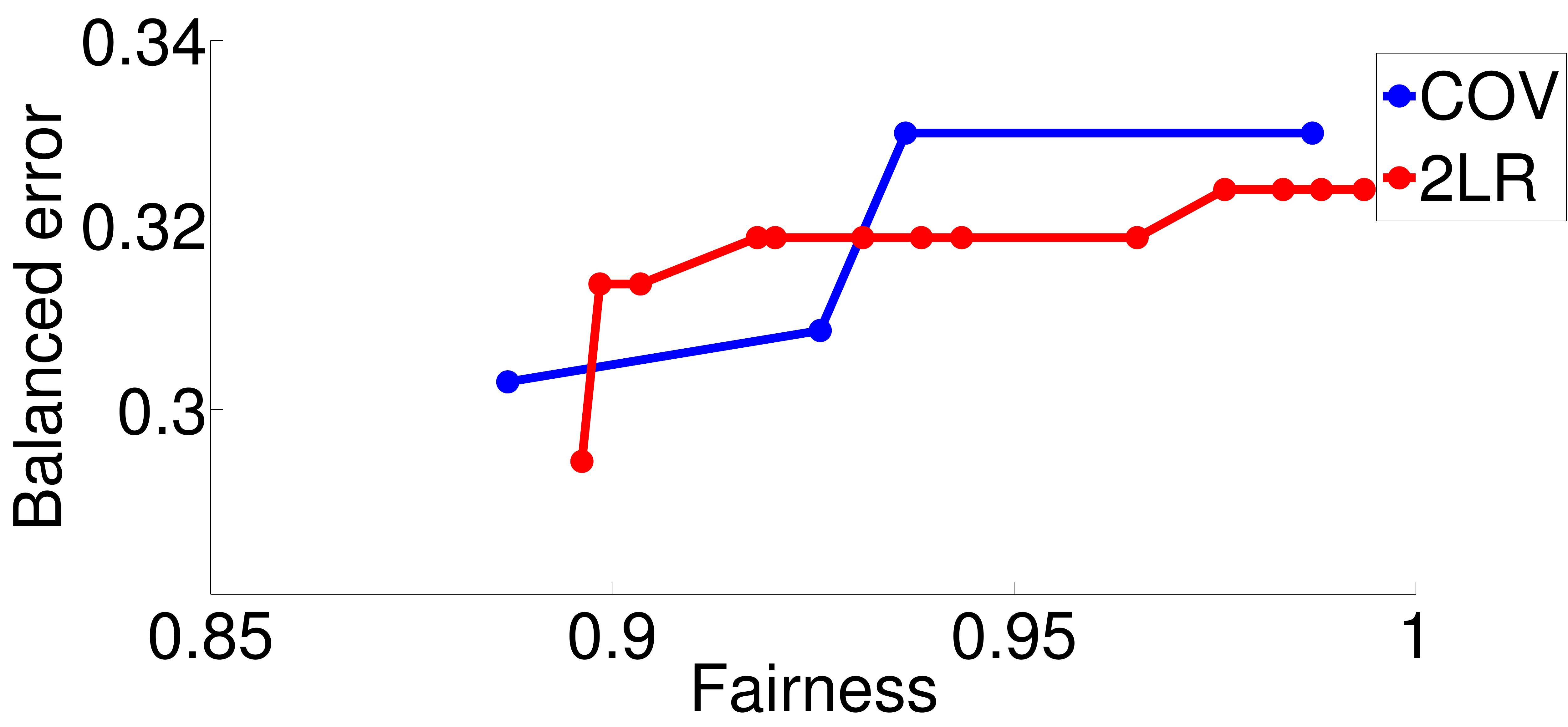} \label{fig:plugin}}
	\end{minipage}\\[-7pt]
    \begin{minipage}[t]{.48\linewidth}
        \caption{Certification of disparate impact via cost-sensitive risk on {\tt german} dataset. Each point represents a different model.}
    \end{minipage}%
    \hfill%
    \begin{minipage}[t]{.48\linewidth}
        \caption{Comparison of plugin ({\tt 2LR}) and {\tt COV} methods as tuning parameters for both are varied.}
    \end{minipage}%
	\vspace{-0.1in}
\end{figure*}





The function $B_{\lambda, c, \bar{c}}(\cdot,\cdot)$ above
has a similar flavour as Equation \ref{eqn:bregman-v1},
and also measures the disalignment of $\eta$ and $\etaSens$
in terms of disagreement around respective thresholds $c, \bar{c}$.
However, there is additionally a dependence on $\lambda$, which depends in some non-trivial manner on $\tau$.
We note that the requirement that $f^*$ be deterministic may be dropped, at the expense of an additional term in Equation \ref{eqn:bregman}
that depends on the alignment of the non-deterministic component and $\eta$.

We make two final comments.
First, the above result holds for both the MD and DI factor, as it only requires the superlevel sets to coincide with those of a cost-sensitive risk.
Second, our notions of (dis)alignment are, roughly, analogous to the notion of compatibility functions in semi-supervised learning \citep{Balcan:2010}, wherein one can guarantee that unlabelled data is useful when there is an alignment of the marginal data distribution with one's function class.

To get an intuitive feel for the disalignment function of Equation \ref{eqn:bregman-v1},
we illustrate the frontier for some simple distributions.
Consider first the case where $\XCal = [-1, 1]$, and $\D$ is such that $\eta( x ) = \indicator{ x > 0 }$ and the marginal over instances is uniform.
Suppose also that $\etaSens( x ) = \indicator{ x > t }$, for $t$ some parameter to be specified.
Consider a cost-sensitive performance and fairness measure with $c = \bar{c} = \half$.
We can explicitly compute the frontier here,
shown in Figure \ref{fig:frontier-indicator}
for a range of $t$.
As $t$ increases, $\eta$ and $\etaSens$ grow increasingly dissimilar, and so the fairness constraint does not affect performance as dramatically:
this is manifest in the fact that $\tau^*$ increases with $t$, as predicted by Equation \ref{eqn:bregman-v1}.
Further, for every $t$,
when there is a tradeoff, it is linear.

Suppose we instead have $\eta( x ) = (1 + \exp(-x))^{-1}$, and retain the same $\etaSens$.
Here again,
Figure \ref{fig:frontier-sigmoid} shows that
as $t$ increases, $\eta$ and $\etaSens$ grow increasingly dissimilar, and we find that $\tau^*$ again increases with $t$.
The impact of changing the shape of $\eta$ is the effect on the frontier when it is nonzero:
as per Equation \ref{eqn:bregman},
this depends on the deviation of $\eta$ from $c$,
and hence the frontier here is nonlinear.


\section{Experiments}
\label{sec:experiments}


We present an experiment
inspired by \citet{Feldman:2015},
who aimed to certify whether a dataset
admits disparate impact
(\ie one can achieve $\DI( f; \DSens ) \leq \tau$ for fixed $\tau$)
by testing if the minimal achievable balanced error
is below some threshold
(see Appendix \ref{sec:di-ber}).
Rather than employ the balanced error,
we follow Lemma \ref{lemm:di-cost-sensitive} and assess the
minimal achievable balanced cost-sensitive risk for $c = (1 + \tau)^{-1}$. 

Specifically, following \citet{Feldman:2015},
we consider the UCI {\tt german} dataset with $\YSens$ denoting whether or not the age of a person is above 25,
and fix $\tau = 0.8$.
For a number of train-test splits to be specified,
we train models to minimise the cost-sensitive logistic loss with parameter $c$ (Equation \ref{eqn:surrogate-cs}),
and evaluate on the test set the disparate impact,
as well as the gap $\Delta( f ) \defEq \CSERRBal( f; \DSens, c ) - (1 - c)$.
Our Lemma \ref{lemm:di-cost-sensitive} indicates that we should find the latter to be positive only when the former is larger than $\tau = 0.8$.

To construct our training sets,
we make an initial 2:1 train-test split of the full data,
treating $\YSens$ as the label to predict. 
To obtain models with varying levels of accuracy in predicting $\YSens$, we inject symmetric label noise of varying rates into the training (but not the test) set.
Figure \ref{fig:certification} shows that for the resulting models,
as per Lemma \ref{lemm:di-cost-sensitive},
there is perfect agreement of disparate impact at $\tau = 0.8$ and $\mathrm{sign}(\Delta( f ))$.


We next present an experiment
analogous to \citet{Zafar:2016},
where on the same {\tt german} dataset
we learn a classifier that respects a symmetrised MD score constraint,
while being accurate for predicting the target variable in the sense of balanced error (BER).
We employ the plugin estimator proposed in \S\ref{sec:bayes-implications},
training logistic regression models to predict the target and sensitive variable
and then combining them via Equation \ref{eqn:bayes-class-sens-unaware} for some $\lambda \in \Real$.
On the test set, we compute the BER for the target variable, and the symmetrised MD score for the sensitive variable.
We then employ the {\tt COV} method of \citet{Zafar:2016}, which uses a surrogate to the MD constraint as discussed in \S\ref{sec:related-bayes},
with tuning parameter $\tau \in \Real_+$ to control the MD score.

Varying $\lambda$ and $\tau$ yields tradeoff curves for both methods.
Figure \ref{fig:plugin} shows these curves at high fairness value,
where we see that our plugin approach is generally competitive with {\tt COV},
resulting in lower BER at higher fairness levels.
(See Appendix \ref{sec:add-expts} for further experiments.)
This illustrates that our Bayes-optimal analysis of Problem \ref{prb:cs-disc-aware}
may be useful in designing fairness-aware classifiers.


\section{Conclusion and future work}

We studied the tradeoffs inherent in the problem of learning with a fairness constraint,
showing that for cost-sensitive fairness measures,
the optimal classifier is an instance-dependent thresholding of the class-probability function,
and 
quantifying the degradation in performance by a measure of alignment of the target and sensitive variable.

There are several interesting directions for future work.
To name a few, we believe it valuable to
study Bayes-optimal scorers for ranking measures such as AUC;
establish consistency of the plugin estimators of \S\ref{sec:bayes-opt};
quantify the impact of working with a finite sample;
and extend our analysis to the case of multi-category sensitive features.

\bibliography{references}
\bibliographystyle{plainnat}

\clearpage

\appendix
\onecolumn


%
{\LARGE
\begin{center}
\textbf{Appendix}
\end{center}
}

\section{Proofs of results in main body}

\let\stdsection\section
\renewcommand\section{\clearpage\stdsection}

%

%
\begin{proof}[Proof of Lemma \ref{lemm:di-cost-sensitive}]
By definition,
\begin{align*}
	\frac{\FPRSens}{1 - \FNRSens} \geq \tau &\iff \FPRSens \geq \tau - \tau \cdot \FNRSens \text{ , since } \FNRSens \leq 1 \\
	&\iff \tau \cdot \FNRSens + \FPRSens \geq \tau \\
	&\iff \frac{\tau}{1 + \tau} \cdot \FNRSens + \frac{1}{1 + \tau} \cdot \FPRSens \geq \frac{\tau}{1 + \tau} \\
	&\iff ( 1 - c ) \cdot \FNRSens + c \cdot \FPRSens \geq 1 - c \\
	&\iff \CSERRBal( f; \DSens, c ) \geq 1 - c.
\end{align*}
The first result follows by definition of disparate impact.

The above may be trivially extended to a symmetrised version of the disparate impact (Equation \ref{eqn:fairness-symm}).
This is since one may equally apply the above to the anti-classifier $1 - f$;
further, we have
$$ \CSERRBal( 1 - f; \D, c ) = 1 - \CSERRBal( f; \D, c ), $$
and so $\DI( 1 - f; \DSens ) \geq \tau \iff \CSERRBal( f; \D, c ) \leq c$,
implying the second result.
\end{proof}

\begin{proof}[Proof of Lemma \ref{lemm:cv-ber}]
By definition,
\begin{align*}
	\MDSens &= 1 - \FPRSens - \FNRSens \\
	&= 1 - 2 \cdot \CSERRBal\left( f; \DSens, \half \right).
\end{align*}
The subsequent implications follow trivially.
\end{proof}

%

%

%
\begin{proof}[Proof of Proposition \ref{prop:bayes-sens-unaware}]
By Lemma \ref{lemm:cs-linear}, the performance measure is
\begin{align*}
\Perf( f; \D, c ) &= (1 - c) \cdot \pi + \E{\X}{ \left( c - \eta( \X ) \right) \cdot f( \X ) }.
\end{align*}
Similarly, the fairness measure is
\begin{align*}
	\Fairness( f; \DDP, \bar{c} ) &= (1 - \bar{c}) \cdot \bar{\pi} + \E{\X}{ \left( \bar{c} - \etaDP( \X ) \right) \cdot f( \X ) }.
\end{align*}
Ignoring constants independent of $f$, the overall objective is thus
\begin{align*}
&\min_f \Perf( f; \D, c ) - \lambda \cdot \Fairness( f; \DDP, \bar{c} ) \\
&=\min_f \E{\X}{ \left( ( c - \eta( \X ) ) - \lambda \cdot ( \bar{c} - \etaDP( \X ) ) \right) \cdot f( \X ) } \\
&= \min_f \E{\X}{ -s^*( x ) \cdot f( \X ) }.
\end{align*}
Thus, at optimality, when $s^*( x ) \neq 0$, $f^*( x ) = \indicator{ s^*( x ) > 0 }$.
\end{proof}

\begin{proof}[Proof of Proposition \ref{prop:bayes-sens-unaware-eoo}]
By Lemma \ref{lemm:cs-linear}, the fairness measure is
\begin{align*}
	\Fairness( f; \DEOO, \bar{c} ) &= (1 - \bar{c}) \cdot \Pr( \YSens = 1 \mid \Y = 1 ) + \E{\X \mid \Y = 1}{ ( \bar{c} - \etaEOO( \X, 1 ) ) \cdot f( \X ) } \\
	&= (1 - \bar{c}) \cdot \Pr( \YSens = 1 \mid \Y = 1 ) + \E{\X}{ \frac{\eta( \X )}{\pi} \cdot ( \bar{c} - \etaEOO( \X, 1 ) ) \cdot f( \X ) },
\end{align*}
where the second line is from applying the importance weighting identity, and the fact that
$$ \frac{\Pr( \X \mid \Y = 1 )}{\Pr( \X )} = \frac{\Pr( \Y = 1 \mid \X = x )}{\Pr( \Y = 1 )} = \frac{\eta( x )}{\pi}. $$
Equivalently, for suitable $\lambda$, we seek
\begin{align*}
& \min_f \E{\X}{ \left( c - \eta( \X ) - \lambda \cdot \frac{\eta( \X )}{\pi} \cdot ( \bar{c} - \etaEOO( \X, 1 ) ) \right) \cdot f( \X ) } \\
&= \min_f \E{\X}{ -s^*( x ) \cdot f( \X ) } \\
\end{align*}
Thus, at optimality, when $s^*( x ) \neq 0$, $f^*( x ) = \indicator{ s^*( x ) > 0 }$.
\end{proof}

\begin{proof}[Proof of Corollary \ref{corr:bayes-sens-aware}]
We simply apply Proposition \ref{prop:bayes-sens-unaware} to $\bar{x} = ( x, \bar{y} )$.
Note that
\begin{align*}
	\etaDP( x, \bar{y} ) &= \Pr( \YSens = 1 \mid \X = x, \YSens = \bar{y} ) \\
	&= \indicator{\bar{y} = 1}.
\end{align*}
Then,
\begin{align*}
f^*( x, \bar{y} ) = 1 &\iff \eta( x, \bar{y} ) > c + \lambda \cdot ( \etaDP( x, \bar{y} ) - \bar{c} ) \\
&\iff \eta( x, \bar{y} ) > c + \lambda \cdot ( \indicator{\bar{y} = 1} - \bar{c} ).
\end{align*}
\end{proof}

\begin{proof}[Proof of Lemma \ref{lemm:constrained-opt}]
By Lemma \ref{lemm:cs-linear}, cost-sensitive risks are linear in the randomised classifier.
In particular, for discrete $\XCal$,
\begin{align*}
\Perf( f; \D ) &= (1 - c) \cdot \pi + \E{\X}{ \left( c - \eta( \X ) \right) \cdot f( \X ) } \\
&= (1 - c) \cdot \pi + \sum_{x \in \XCal}{ m( x ) \cdot \left( c - \eta( x ) \right) \cdot f( x ) },
\end{align*}
where $m( x ) = \Pr( \X = x )$.
Similarly,
\begin{align*}
\Fairness( f; \DSens ) &= (1 - \bar{c}) \cdot \pi + \E{\X}{ \left( \bar{c} - \etaSens( \X ) \right) \cdot f( \X ) } \\
&= (1 - \bar{c}) \cdot \pi + \sum_{x \in \XCal}{ m( x ) \cdot \left( \bar{c} - \etaSens( x ) \right) \cdot f( x ) }.
\end{align*}
Now let
\begin{align*}
	( \forall x \in \XCal ) \, a( x ) &\defEq m( x ) \cdot \left( c - \eta( x ) \right) \\
	( \forall x \in \XCal ) \, b( x ) &\defEq m( x ) \cdot \left( \bar{c} - \etaSens( x ) \right).
\end{align*}
Then, the optimisation is
\begin{align*}
	\min_f a^T f \colon &-b^T f \leq -\tau \\
	&b^T f \leq 1 - \tau \\
	&0 \leq f \leq 1.
\end{align*}
This is a linear objective with linear constraints.
We thus may find the optimal random classifier by the solution to a linear program.
\end{proof}

%

%
\begin{proof}[Proof of Lemma \ref{lemm:convex-frontier}]
We wish to determine whether, for any $\tau, \tau' \in \Real_+$ and $\lambda \in [0, 1]$,
\begin{align*}
	&F( \lambda \tau + (1 - \lambda) \tau' ) \stackrel{?}{\leq} \lambda F( \tau ) + (1 - \lambda) F( \tau' ) \\
	&\iff \Perf( f^*_{\lambda \tau + (1 - \lambda) \tau'} ) \stackrel{?}{\leq} \lambda \Perf( f^*_\tau ) + (1 - \lambda) \Perf( f^*_{\tau'} ) \\
	&\impliedby \Perf( f^*_{\lambda \tau + (1 - \lambda) \tau'} ) \stackrel{?}{\leq} \Perf( \lambda f^*_\tau + (1 - \lambda) f^*_{\tau'} ) \text{ for convex } \Perf \\
	&\impliedby \lambda f^*_\tau + (1 - \lambda) f^*_{\tau'} \text{ feasible for } \lambda \tau + (1 - \lambda) \tau' \\
	&\impliedby \Fairness(\lambda f^*_\tau + (1 - \lambda) f^*_{\tau'} ) \stackrel{?}{\geq} \lambda \tau + (1 - \lambda) \tau'.
\end{align*}
By definition, $f^*_\tau, f^*_\tau$ must be feasible for their corresponding problems, and so
\begin{align*}
	\Fairness( f^*_\tau    ) &\geq \tau \\
	\Fairness( f^*_{\tau'} ) &\geq \tau'.
\end{align*}
We thus want to determine whether
\begin{align*}
	&\Fairness(\lambda f^*_\tau + (1 - \lambda) f^*_{\tau'} ) \geq \lambda \tau + (1 - \lambda) \tau' \\
	&\impliedby \Fairness(\lambda f^*_\tau + (1 - \lambda) f^*_{\tau'} ) \geq \lambda \Fairness( f^*_\tau ) + (1 - \lambda) \Fairness( f^*_{\tau'} ) \\
	&\impliedby \Fairness \text{ concave. }
\end{align*}
Since $\Fairness$ is cost-sensitive, it is linear by Lemma \ref{lemm:cs-linear}, and hence concave.
The result thus follows.
\end{proof}

\begin{proof}[Proof of Proposition \ref{prop:frontier-cutoff}]
As argued in the body,
$\tau^* = \FairnessSymm( f^*_0; \DSens )$.
Thus, it remains to compute this quantity.
First, note that we may pick $f^*_0 = \indicator{ \eta( x ) > c }$ by Lemma \ref{lemm:cs-opt}.\footnote{The behaviour of optimal solutions at $\eta( x ) = c$ is unconstrained; however, different choices can in fact lead to different fairness values. Nonetheless, by picking a specific optimal solution, we are nonetheless guaranteed that picking $\tau^*$ as defined will yield zero risk.}
Next, observe that since $\Fairness$ is a cost-sensitive risk,
\begin{align*}
\Fairness( f; \DSens ) &= \Fairness( f; \DSens ) - \min_g \Fairness( g; \DSens ) + \min_g \Fairness( f; \DSens ) \\
&= \CSERR( f; \DSens, \bar{c} ) - \min_g \CSERR( g; \DSens, \bar{c} ) + \min_g \CSERR( f; \DSens, \bar{c} ) \\
&= \E{\X}{ ( c - \etaSens( \X ) ) \cdot ( f( \X ) - \indicator{ \etaSens( \X ) > c } ) } + \min_g \CSERR( f; \DSens, \bar{c} ) \text{ by Lemma \ref{lemm:cs-regret} }.
\end{align*}
Plugging in $f^*_0$, and applying the second statement of Lemma \ref{lemm:cs-regret},
\begin{align*}
\Fairness( f; \DSens ) &= \E{\X}{ | \etaSens( \X ) - c | \cdot \indicator{ ( \etaSens( \X ) - \bar{c} ) \cdot ( \eta( X ) - c ) < 0 } } + \min_g \CSERR( f; \DSens, \bar{c} ).
\end{align*}
Now we just apply Lemma \ref{lemm:cs-f-div} to the second term.
\end{proof}

\begin{proof}[Proof of Proposition \ref{prop:bregman}]
Observe that $f^*_0 \in \argmin{}{ \Perf( f; \D ) }$.
Consequently, the frontier may be re-written
$$ F( \tau ) = \regret( f^*_\tau; \D ) $$
where the \emph{regret} or \emph{excess risk} of a classifier is
$$ \regret( f; \D ) = \Perf( f; \D ) - \min_{g \colon \XCal \to [0, 1]} \Perf( g; \D ). $$
This lets us specify the form of $F( \cdot )$ analytically
when $\Perf$ is a cost-sensitive risk.
By Lemma \ref{lemm:cs-regret},
\begin{align*}
	F( \tau ) &= \E{\X \sim M}{ ( c - \eta( \X ) ) \cdot ( f^*_\tau( \X ) - \indicator{\eta( \X ) > c} )}.
\end{align*}
Now, since $f^*_\tau$ is the solution to a linear program by Lemma \ref{lemm:constrained-opt},
we can appeal to strong duality (see Appendix \ref{sec:lagrangian}) to conclude that there exists some $\lambda$ for which the corresponding soft-constrained version of the problem (Equation \ref{eqn:cs-fairness-lagrange})
has the same optimal value.
This means there is some Bayes-optimal classifier $f^*_\lambda$ to Equation \ref{eqn:cs-fairness-lagrange} for which
\begin{align*}
	F( \tau ) &= \E{\X \sim M}{ ( c - \eta( \X ) ) \cdot ( f^*_\lambda( \X ) - \indicator{\eta( \X ) > c} )}.
\end{align*}
Now, if additionally this $f^*_\lambda$ is deterministic,
then also by Lemma \ref{lemm:cs-regret},
$$	F( \tau ) = \E{\X \sim M}{ | \eta( \X ) - c | \cdot \indicator{ (\eta( \X ) - c) \cdot (2 f^*_\lambda( \X ) - 1) < 0} }. $$
Now just plug in the definition of $f^*_\lambda$ from Equations \ref{eqn:bayes-scorer-sens-unaware}, \ref{eqn:bayes-class-sens-unaware}.
\end{proof}

\section{Helper results}

\begin{lemma}
\label{lemm:cs-linear}
Pick any distribution $\D$ and randomised classifier $\randClass$.
Then, for any cost parameter $c \in [0, 1]$,
\begin{align*}
	\CSERR( f; \D, c ) &= (1 - c) \cdot \pi + \E{\X}{ \left( c - \eta( \X ) \right) \cdot f( \X ) }
\end{align*}
where $\pi = \Pr( \Y = 1 )$, $\eta( x ) = \Pr( \Y = 1 \mid \X = x )$.
\end{lemma}

\begin{proof}[Proof of Lemma \ref{lemm:cs-linear}]
By definition,
\begin{align*}
\CSERR( f; \D, c ) &= (1 - c) \cdot \pi \cdot \E{\X \mid \Y = 1}{ 1 - f( \X ) } + c \cdot (1 - \pi) \cdot \E{\X \mid \Y = 0}{ f( \X ) } \\
&= \E{\X}{ (1 - c) \cdot \eta( \X ) \cdot (1 - f( \X )) + c \cdot (1 - \eta( \X )) \cdot f( \X ) } \\
&= \E{\X}{ (1 - c) \cdot \eta( \X ) } + \E{\X}{ \left( c \cdot (1 - \eta( \X )) - (1 - c) \cdot \eta( \X ) \right) \cdot f( \X ) } \\
&= (1 - c) \cdot \pi + \E{\X}{ \left( c \cdot (1 - \eta( \X )) - (1 - c) \cdot \eta( \X ) \right) \cdot f( \X ) } \\
&= (1 - c) \cdot \pi + \E{\X}{ \left( c - \eta( \X ) \right) \cdot f( \X ) }. 
\end{align*}
The second line is since $\Pr( \X \mid \Y = 1 ) \cdot \Pr( \Y = 1 ) = \Pr( \X ) \cdot \Pr( \Y = 1 \mid \X )$.
\end{proof}

\begin{lemma}
\label{lemm:cs-opt}
Pick any distribution $\D$ and cost parameter $c \in [0, 1]$.
Let
$$ ( \forall x \in \XCal ) \, s^*( x ) = \eta( x ) - c $$
where $\eta( x ) = \Pr( \Y = 1 \mid \X = x )$.
Then, any randomised classifier $f^*$ satisfying
$$ ( \forall x \in \XCal ) \, s^*( x ) \neq 0 \implies f^*( x ) = \indicator{ s^*( x ) > 0 } $$
minimises $\CSERR( f; \D, c )$.
\end{lemma}

\begin{proof}[Proof of Lemma \ref{lemm:cs-opt}]
By Lemma \ref{lemm:cs-linear},
we need to find,
for each $x \in \XCal$
$$ \min_{f( x ) \in [0, 1]} ( c - \eta( x ) ) \cdot f( x ) = \min_{f( x ) \in [0, 1]} -s^*( x ) \cdot f( x ), $$
observing that the minimisation may be done pointwise.
Clearly, it is optimal to predict $f^*( x ) = 1$ when $s^*( x ) > 0$, and $f^*( x ) = 0$ when $\eta( x ) < c$.
When $\eta( x ) = c$, any prediction is optimal.
\end{proof}

\begin{lemma}
\label{lemm:cs-regret}
Pick any distribution $\DFull$ and cost parameter $c \in [0, 1]$.
Then,
for any randomised classifier $f$,
$$ \CSERR( f; \D, c ) - \min_{g \colon \XCal \to [0, 1]} \CSERR( g; \D, c ) = \E{\X}{ ( c - \eta( \X ) ) \cdot ( f( \X ) - \indicator{ \eta( \X ) > c } ) }. $$
If further $f \in \{ 0, 1 \}^{\XCal}$,
$$ \CSERR( f; \D, c ) - \min_{g \colon \XCal \to [0, 1]} \CSERR( g; \D, c ) = \E{\X}{ | \eta( \X ) - c | \cdot \indicator{ (\eta( \X ) - c) \cdot (2 f( x ) - 1 ) < 0 } }. $$
\end{lemma}

\begin{proof}[Proof of Lemma \ref{lemm:cs-regret}]
By Lemma \ref{lemm:cs-opt}, an optimal classifier for $\Perf( g; \D )$ is the deterministic $f^*( x ) = \indicator{ \eta( x ) > c }$.
Thus, plugging this into Lemma \ref{lemm:cs-linear},
$$ \CSERR( f; \D, c ) - \min_{g \colon \XCal \to [0, 1]} \CSERR( g; \D, c ) = \E{\X}{ \left( c - \eta( \X ) \right) \cdot \left( f( \X ) - \indicator{ \eta( \X ) > c } \right) }. $$
The second statement follows from a simple case analysis.
The difference $f( x ) - \indicator{ \eta( x ) > c }$ takes on the value $+1$ when $f( x ) = 1$ and $\eta( x ) < c$, and $-1$ when $f( x ) = 0$ and $\eta( x ) > c$,
\ie the value $\mathrm{sign}( c - \eta( x ) )$ when $2 f - 1$ and $\eta( x ) - c$ disagree in sign.
Since $| z | = z \cdot \mathrm{sign}( z )$, the result follows.
\end{proof}

\begin{lemma}
\label{lemm:cs-f-div}
Pick any distribution $\DFull$ and cost parameter $c \in [0, 1]$.
Then,
$$ \min_{\randClassifier} \CSERR( f; \D, c ) = -\mathbb{I}_\varphi( \Pr( \X \mid \Y = 1 ), \Pr( \X \mid \Y = 0 ) ) $$
where
$\mathbb{I}_f( \cdot, \cdot )$ denotes the $f$-divergence between distributions,
and
$$ \varphi( t ) = -\min\left( (1 - c) \cdot {\pi \cdot t}, c \cdot {(1 - \pi)} \right). $$
\end{lemma}

\begin{proof}[Proof of Lemma \ref{lemm:cs-f-div}]
This follows from \citet[Theorem 9]{Reid:2011}, applied as follows.
Let $f^*( x ) = \indicator{ \eta( x ) > c } + \frac{1}{2} \cdot \indicator{ \eta( x ) = c }$,
which is an optimal classifier for $\CSERR( f; \D, c )$
by Lemma \ref{lemm:cs-opt}.
Then,
$$ \CSERR( f^*; \D, c ) = \E{(\X, \Y) \sim \D}{ \ell( \Y, \eta( \X ) ) } $$
where $\ell$ is the \emph{cost-sensitive loss} given by
\begin{align*}
	\ell( 1, v ) &= (1 - c) \cdot \left( \indicator{ v < c } + \frac{1}{2} \indicator{v = c} \right) \\
	\ell( 0, v ) &= c 		\cdot \left( \indicator{ v > c } + \frac{1}{2} \indicator{v = c} \right).
\end{align*}
Now, $\ell$ is \emph{proper} in the sense of \citet{Reid:2010}.
Consequently,
$$ \E{(\X, \Y) \sim \D}{ \ell( \Y, \eta( \X ) ) } = \min_{\hat{\eta} \colon \XCal \to [0, 1]} \E{(\X, \Y) \sim \D}{ \ell( \Y, \hat{\eta}( \X ) ) }. $$
The right hand side above is the \emph{Bayes-risk} for the proper loss $\ell$ in the sense of \citet{Reid:2011}.
Consequently,
by \citet[Theorem 9]{Reid:2011},
$$ \E{(\X, \Y) \sim \D}{ \ell( \Y, \eta( \X ) ) } = - \mathbb{I}_\varphi( \Pr( \X \mid \Y = 1 ), \Pr( \X \mid \Y = 0 ) ), $$
where
\begin{align*}
	\varphi( t ) &= -\min\left( (1 - c) \cdot {\pi \cdot t}, c \cdot {(1 - \pi)} \right). 
\end{align*}
This may be verified easily, since
\begin{align*}
	\CSERR( f^*; \D, c ) &= \E{\X}{ (1 - c) \cdot \eta( \X ) \cdot \left( \indicator{ \eta( x ) < c } + \frac{1}{2} \cdot \indicator{ \eta( x ) = c } \right) + c \cdot (1 - \eta( \X )) \cdot \left( \indicator{ \eta( x ) > c } + \frac{1}{2} \cdot \indicator{ \eta( x ) = c } \right) } \\
	&= \E{\X}{ \min( (1 - c) \cdot \eta( \X ), c \cdot (1 - \eta( \X ) ) ) },
\end{align*}
while,
if $P = \Pr( \X \mid \Y = 1 ), Q = \Pr( \X \mid \Y = 0 )$
with densities $p, q$,
\begin{align*}
	-\mathbb{I}_\varphi( P, Q ) &= -\E{\X \sim Q}{ \varphi\left( \frac{p(\X)}{q(\X)} \right) } \\
	&= \E{\X \sim Q}{ \min\left( (1 - c) \cdot {\pi \cdot \frac{p(\X)}{q(\X)}}, c \cdot {(1 - \pi)} \right) } \\
	&= \E{\X \sim M}{ \min\left( (1 - c) \cdot \pi \cdot \frac{p(\X)}{m(\X)}, c \cdot (1 - \pi) \cdot \frac{q(\X)}{m(\X)} \right) } \\
	&= \E{\X \sim M}{ \min\left( (1 - c) \cdot \eta( \X ), c \cdot (1 - \eta( \X )) \right) } \\
	&= \CSERR( f^*; \D, c ).
\end{align*}
\end{proof}

\begin{corollary}
\label{corr:bayes-sens-aware-eoo}
Pick any distribution $\DFull$, costs $c, \bar{c} \in [0, 1]$, and $\lambda \in \Real$.
Let
\begin{align*}
 s^*( x, 0 ) &= \left( 1 + \lambda \cdot \pi^{-1} \cdot \bar{c} \right) \cdot \eta( x, 0 ) - c \\
 s^*( x, 1 ) &= \left( 1 - \lambda \cdot \pi^{-1} \cdot (1 - \bar{c}) \right) \cdot \eta( x, 1 ) - c
\end{align*}
where $\eta( x, \bar{y} ) = \Pr( \Y = 1 \mid \X = x, \YSens = \bar{y} )$
Then,
\begin{align*}
	\argmin{f \in [0,1]^{\XCal}}{ R( f; \D, \DEOO, c, \bar{c}, \lambda ) } = \bigl\{ &f^* \mid (\forall x \in \XCal) \, s^*( x ) \neq 0 \implies f^*( x ) = \indicator{ s^*( x ) > 0} \bigl\}.
\end{align*}

\end{corollary}

\begin{proof}[Proof of Corollary \ref{corr:bayes-sens-aware-eoo}]
Plug in $\etaSens( x, \bar{y} ) = \indicator{\bar{y} = 1}$ into Proposition \ref{prop:bayes-sens-unaware-eoo}.
\end{proof}

\section{Symmetrised fairness when $\piSens \neq \half$}
\label{sec:asymm-fairness}

The presentation of symmetrised in the body
noted that one needs to guard against anti-classifiers.
However, the proposal
$$	\FairnessSymm( f; \DSens ) = \Fairness( f; \DSens ) \land \Fairness( 1 - f; \DSens ) $$
is maximised when using the constant classifier $f = \half$.
It is preferable for the fairness measure to instead be maximised
when $f = \piSens$,
which would be the optimal prior prediction before we get a chance to look at the data.
In fact, this is easy to achieve by modifying 
$$	\FairnessSymm( f; \DSens ) = ( (1 - \alpha) \cdot \Fairness( f; \DSens ) ) \land ( \alpha \cdot \Fairness( 1 - f; \DSens ) ) $$
where $\alpha = \Fairness( \piSens; \DSens )$,
so that we asymmetrically penalise raw fairness scores below and above those of $\Fairness( \piSens; \DSens )$.
Note that this would simply introduce additional asymmetric scalings into our results,
\eg for bounds relating the symmetrised disparate impact to a cost-sensitive risk (Lemma \ref{lemm:di-cost-sensitive}).

\section{Relating the constrained and unconstrained objectives}
\label{sec:lagrangian}

Consider the constrained version of the fairness problem,
$$ f^* \in \argmin{f \in [0, 1]^{\XCal}}{\Perf( f; \D )} \colon \FairnessSymm( f; \DSens ) \geq \tau. $$
By Lemma \ref{lemm:constrained-opt}, for finite $\XCal$, this is expressible as the solution to a linear program
\begin{align*}
	\min_{f \in \FCal} a^T f
\end{align*}
where
$$ \FCal = \{ f \mid b^T f \in [ \tau,  1 - \tau ], 0 \leq f( x ) \leq 1 \} $$
and
\begin{align*}
	( \forall x \in \XCal ) \, a( x ) &\defEq m( x ) \cdot \left( c - \eta( x ) \right) \\
	( \forall x \in \XCal ) \, b( x ) &\defEq m( x ) \cdot \left( \bar{c} - \etaSens( x ) \right).
\end{align*}
Now, by strong duality for linear programs\footnote{This implicitly assumes feasibility of the primal problem, \ie that we pick $\tau$ such that it is possible to find a randomised classifier with symmetrised fairness at least $\tau$.}, we have
\begin{equation}
	\label{eqn:strong-dual}
	\min_{f \in \FCal} a^T f = \max_{\lambda_1, \lambda_2 \geq 0} \left( \min_{f \in [ 0, 1 ]^{\XCal}} ( a - \lambda_1 b + \lambda_2 b)^T f \right )+ \lambda_1 \tau - \lambda_2 (1 - \tau).
\end{equation}
Observe now that the inner optimisation is
\begin{align*}
	&\min_{f \in [ 0, 1 ]^{\XCal}} ( a - \lambda_1 b + \lambda_2 b)^T f \\
	&= \min_{f \in [ 0, 1 ]^{\XCal}} ( a - (\lambda_1 - \lambda_2) b )^T f \\
	&= \min_{f \in [ 0, 1 ]^{\XCal}} \sum_{x \in \XCal} m( x ) \cdot \left[ c - \eta( x ) - (\lambda_1 - \lambda_2) \left( \bar{c} - \etaSens( x ) \right) \right] \cdot f( x ) \\
	&= \min_{f \in [ 0, 1 ]^{\XCal}} \E{\X}{ \left( c - \eta( \X ) - (\lambda_1 - \lambda_2) \left( \bar{c} - \etaSens( \X ) \right) \right) \cdot f( \X ) } \\
	&= \min_{f \in [ 0, 1 ]^{\XCal}} \CSERR( f; \D, c ) - ( \lambda_1 - \lambda_2 ) \cdot \CSERR( f; \DSens, \bar{c} ).
\end{align*}
That is, we solve Equation \ref{eqn:cs-fairness-lagrange} for $\lambda = \lambda_1 - \lambda_2$.
By sweeping over $\lambda$,
we can thus in principle find the one which achieves the highest value of the objective in Equation \ref{eqn:strong-dual},
and consequently find the solution to the constrained problem for a fixed $\tau$.

Note that strong duality guarantees agreement of the objective functions.
In general, it does not mean that every optimal solution to the inner problem (for optimal $\lambda_1, \lambda_2$)
will also be optimal for the original constrained problem.
As an extreme case, suppose that $\etaSens = \eta$, and $c = \bar{c}$.
Then, the constrained problem has optimal solution any $f$ for which $\CSERR( f; \DSens, \bar{c} ) = \tau$,
so that the frontier is linear.
On the other hand, we will find that for the optimal $\lambda_1, \lambda_2$,
the inner optimisation is simply of the constant $0$,
in which case every $f$ is deemed optimal.

\section{A survey of fairness in philosophy and welfare economics}
\label{sec:philosophy}

Fairness lies at the heart of justice, which is ``the first virtue of
social institutions'' \citep[p. 586]{Rawls:1971aa}.
But what is fairness? Rawls develops a theory of fairness
utilizing the ``veil of ignorance'' whereby a person's position in society
is held to be unknown while designing the rules of a just society.  The
analogue in machine learning is that membership of a given category should
not be used, one way or another, in determining outcomes for an individual.
Analogous theories such as that of \citep{Harsanyi:1955aa} and
\citep{Sen:2009aa} differ in what ignorance means (uniform prior over what
role one has, or the perspective of a separate impartial observer). In all
cases, the general idea is that a just outcome should not depend (either
way) on membership of a particular category, but should focus upon the
individual; \cite{Rawls:1971aa} motivates his approach by saying
``utilitarianism does not take seriously the distinction between persons''
(page 26) because it is only concerned with \emph{average} welfare.  While
these ideas have had profound impact on political philosophy, their
translation into mathematical theories is lacking. With few exceptions
\citep{Bimore:1994aa,Binmore:2005aa}, there is little formal utilitarian
literature that grapples with fairness.

Rawls argues the advantages of ``pure procedural justice'' where the
attainment of justice (hence fairness) is a consequence of the process
followed, not the outcome obtained. This is, of course, taken for granted
in machine learning and statistics, where one analyses a procedure, not the
results of the procedure on a given set of data. Taking this principle
seriously means that protected attributes have to be identified ahead of
time, and not after the fact because you did not like the outcome; confer
the analogous problem in designing electoral districts: getting the process
right makes the problem straightforward \citep{Vickrey:1961aa}; attempting
to judge fairness \emph{a posteriori} is a mess, with infinite arguments
possible about whether a boundary is ``bizarre'' \citep{Chambers:2010aa}.

The theory we present here 
follows the precept of \citet[Chapter 18]{Sen:2009aa} that mere
identification of ``fully just social arrangements is neither necessary
nor sufficient.'' We embrace Sen's pragmatism by focussing on the
\emph{quantifiable tradeoffs} one might make to approach (certain notions
of) fairness and hence
justice. Rawls acknowledges that there will be tradeoffs between overall
social utility and fairness (pp37ff) but tries to argue what the ``right''
tradeoff is. We do not try to solve that question, believing instead there
is unlikely to be a universal right tradoff, and instead focus upon merely
quantifying what the tradeoffs might be. Focussing upon quantifying the
traedoff between utility and fairness has only very recently drawn
attention in the machine learning literature \citep{Johnson:2016aa}.

The approaches to fairness in the machine learning literature, which we
follow in this paper, focusses on the notion of a protected attribute, and
assumes that both its choice is manifest, and indeed it is a sensible
categorisation (e.g. notions of race). One should take care with this
though, because any
such categories to which people are assigned to are not
intrinsic to the world, but another choice that we make
\citep{Lakoff:1987aa}, which can be highly ambiguous and contested
\citep{Bowker:1999aa}. 

Singling out \emph{particular} attributes to
be protected (as opposed to the rather more sweeping requirements of Rawls'
full theory that requires no specific attributes be taken into account) 
opens the door to the problems inherent in the ``ecological
fallacy'' \citep{Kramer:1983aa} --- making inferences about individuals
based on membership in a category, which is precisely what some argue one
should not do \citep{Lippert-Rasmussen:2011aa}.

Indeed the very notions of fairness studied in this paper
glaringly fails the test against ``discrimination'' if one adopts a
standard definition (e.g. from Wikipedia):
\begin{quotation}
	In human social affairs, discrimination is treatment or
	consideration of, or making a distinction in favour or against a
	person or thing based on the group, class, or category to which
	that person or thing is perceived to belong rather than on
	individual merit.
\end{quotation}
By that definition, the only non-discriminatory approach is to ignore the
protected attributes \emph{entirely} and take the outcome as it comes.
Reconciling this tension remains a challenge!

\section{Relating disparate impact and balanced error}
\label{sec:di-ber}

Following \citet{Feldman:2015}, we explore the relationship between the balanced error and disparate impact.
Intuitively, we expect that when the balanced error of a classifier is low
-- meaning that the classifier accurately predicts the sensitive variable --
we will have disparate impact.
Conversely, we might hope that possessing disparate impact implies a low balanced error.
Can we formalise a relationship akin to Lemma \ref{lemm:cv-ber}?

We have the following relations between the two quantities.
In what follows, let
$$ \BERSens = \frac{\FPRSens + \FNRSens}{2}. $$

\begin{lemma}
\label{lemm:ber-di-trivial}
Pick any distribution $\DSens$
and randomised classifier $\randClassifier$
with $\FNRSens \neq 1$.
Then,
\begin{equation}
	\label{eqn:di-to-ber}
	\begin{aligned}
		\DISens &= \frac{\FPRSens}{1 - 2 \cdot \BERSens + \FPRSens} \\
		&= \frac{2 \cdot \BERSens - \FNRSens}{1 - \FNRSens},
	\end{aligned}
\end{equation}
and similarly
\begin{equation}
	\label{eqn:ber-to-di}
	\begin{aligned}
		\BERSens &= \frac{1}{2} \cdot \FNRSens + \frac{1}{2} \cdot (1 - \FNRSens) \cdot \DISens \\
		&= \frac{1}{2} \cdot \FPRSens + \frac{1}{2} \cdot \left( 1 - \frac{\FPRSens}{\DISens} \right).
	\end{aligned}
\end{equation}
\end{lemma}

\begin{proof}[Proof of Lemma \ref{lemm:ber-di-trivial}]
These are trivial consequences of the fact that, by definition of $\DISens$ (Equation \ref{eqn:di-fpr}),
$$ \FPRSens = (1 - \FNRSens) \cdot \DISens. $$
\end{proof}

We now turn to relating a bound on the balanced error to a bound on the disparate impact factor.
The following is a minor generalisation of \citet[Theorem 4.1]{Feldman:2015}
to account for disparate impact at any level.

\begin{lemma}
\label{lemm:ber-di}
Pick any distribution $\DSens$
and randomised classifier $\randClassifier$
with $\FNRSens \neq 1$.
Then, for any $\epsilon \in [0, \frac{1}{2}]$,
\begin{align*}
	\BERSens \leq \epsilon \iff \DISens &\leq \frac{\FPRSens}{1 - 2\cdot\epsilon + \FPRSens} \land \frac{2 \cdot \epsilon - \FNRSens}{1 - \FNRSens},
\end{align*}
and for any $\tau \in [0, 1]$,
$$ \DISens \leq \tau \iff \BERSens \leq \left( \frac{\tau}{2} + \frac{1 - \tau}{2} \cdot \FNRSens \right) \land \left( \frac{1}{2} - \frac{1 - \tau}{2 \cdot \tau} \cdot \FPRSens \right). $$
\end{lemma}

\begin{proof}[Proof of Lemma \ref{lemm:ber-di}]
The first equivalence follows from the two expressions in Equation \ref{eqn:di-to-ber},
and the fact that the dependence on $\BERSens$
is monotone increasing.

This second equivalence follows from the two expressions in Equation \ref{eqn:ber-to-di},
and the fact that the dependence on $\DISens$
is monotone increasing.
\end{proof}

When $\tau = 0.8$, this means that
$$ \DISens \leq 0.8 \iff \BERSens \leq \left( \frac{2}{5} + \frac{1}{10} \cdot \FNRSens \right) \land \left( \frac{1}{2} - \frac{1}{8} \cdot \FPRSens \right). $$

\subsection{Low balanced error implies disparate impact}

It is of interest to remove the dependence of the above bounds on the false positive and negative rates of $f$.
For one direction, this is possible.

\begin{corollary}
\label{corr:ber-di}
Pick any distribution $\DSens$
and randomised classifier $\randClassifier$
with $\FNRSens \neq 1$.
Then, for any $\epsilon \in [0, \frac{1}{2}]$,
\begin{align*}
	\BERSens \leq \epsilon \implies \DISens &\leq 2 \cdot \epsilon,
\end{align*}
or for any $\tau \in [0, 1]$,
$$ \DISens \geq \tau \implies \BERSens \geq \frac{\tau}{2}. $$
\end{corollary}

\begin{proof}
The first bound follows from Lemma \ref{lemm:ber-di} and the fact that if $\BERSens \leq \epsilon$,
it must be true that $\FPRSens \lor \FNRSens \leq 2 \cdot \epsilon$.
The second bound is the contrapositive of the first.
\end{proof}

Corollary \ref{corr:ber-di} says that with a balanced error of $\frac{\tau}{2}$ or less, we are guaranteed a disparate impact of level at least $\tau$, though possibly worse.
So, if we want to guarantee a lack of disparate impact at level $\tau$, it is \emph{necessary} that the balanced error be at least $\frac{\tau}{2}$.
But is this condition also \emph{sufficient}?
Unfortunately, it is not.

\subsection{Disparate impact does not imply low balanced error}

It is evident from Lemma \ref{lemm:ber-di} that regardless of the precise level of impact $\tau$,
we \emph{could} have a classifier with balanced error arbitrarily close to $\frac{1}{2}$.
The basic issue is that by driving the false positive rate to $0$,
we trivially have disparate impact.
By further driving the false negative rate to $0$ (\ie by predicting everything negative),
we trivially have a balanced error rate of $\frac{1}{2}$.

\begin{corollary}
\label{corr:ber-no-di}
Pick any distribution $\DSens$.
Then, for any $\tau \in [0, 1]$
there exists a classifier $f \colon \XCal \to \Labels$
with
$$ \BERSens = \frac{1}{2} $$
$$ \DISens \leq \tau. $$
\end{corollary}

\begin{proof}
Consider the trivial classifier with
$$ \FPRSens = 0 $$
$$ \FNRSens = 1. $$
Clearly, this has balanced error $\frac{1}{2}$.
Evidently, this classifier also has disparate impact at level $\tau$.
\end{proof}

Corollary \ref{corr:ber-no-di} says that even if we have a classifier with high balanced error, there is no guarantee it will not have disparate impact.
This is a worst case analysis over all possible classifiers we \emph{might} have obtained.
However, if we happen to know the false positive and negative rates we \emph{actually} have obtained, we might be able to conclude there is no disparate impact.
This is used in \citet[Section 4.2]{Feldman:2015} to certify the lack of disparate impact for a particular classifier.

\begin{corollary}
\label{corr:ber-no-di-cheat}
Pick any distribution $\DSens$
and randomised classifier $\randClassifier$.
For any $\tau \in [0, 1]$,
$$ \BERSens \geq \left( \frac{\tau}{2} + \frac{1 - \tau}{2} \cdot \FNRSens \right) \land \left( \frac{1}{2} - \frac{1 - \tau}{2 \cdot \tau} \cdot \FPRSens \right) \iff \DISens \geq \tau. $$
\end{corollary}

\begin{proof}[Proof of Corollary \ref{corr:ber-no-di-cheat}]
This is the contrapositive of Lemma \ref{lemm:ber-di}.
\end{proof}

\section{Additional experiments}
\label{sec:add-expts}

We present a further experiment on the synthetic dataset considered in \citet{Zafar:2016},
where $\Pr( \Y = 1 ) = 0.5$,
each $\X \mid \Y = y \sim \mathscr{N}( \mu_y, \Sigma_y )$ where
\begin{align*}
	\mu_1 	   &=  \begin{bmatrix} 2 & 2 \end{bmatrix} \\
	\Sigma_1   &=  \begin{bmatrix} 5 & 1 \\ 1 & 5 \end{bmatrix} \\
	\mu_{0}    &=  \begin{bmatrix} 10 & 1 \\ 1 & 3 \end{bmatrix} \\
	\Sigma_{0} &=  \begin{bmatrix} -2 & -2 \end{bmatrix},
\end{align*}
and
$$ \Pr( \YSens = 1 \mid \X = x ) = \frac{\Pr( \X = R x \mid \Y = 1 )}{\Pr( \X = R x \mid \Y = 1 ) + \Pr( \X = R x \mid \Y = -1 )} $$
for rotation matrix $R = \begin{bmatrix} \cos \phi & -\sin \phi \\ \sin \phi & \cos \phi \end{bmatrix}$.
We pick $\phi = 0.5$.

We generated $N = 10^{4}$ samples from this distribution,
and followed the same setup as the body:
we construct a 2:1 train-test split,
and compare {\tt COV} and our plugin ({\tt 2LR}) approach
in terms of the balanced error of predicting $\Y$, versus the MD score in predicting $\YSens$.

Figure \ref{fig:zafar-gaussian} shows the tradeoff curves of the methods closely track each other.
However, the plugin approach performs slightly worse at higher fairness levels.
We conjecture this is due to the fact that logistic regression is not suitable for $\eta, \etaSens$,
as the class-conditionals have non-isotropic covariance and thus possess \emph{quadratic} boundaries.
When we explicitly include quadratic features as input to both methods, Figure \ref{fig:zafar-gaussian-quadratic} shows that the plugin approach performs slightly better than {\tt COV}.

\begin{figure}[!h]
	\centering
	\subfigure[Raw features.]{\includegraphics[scale=0.0925]{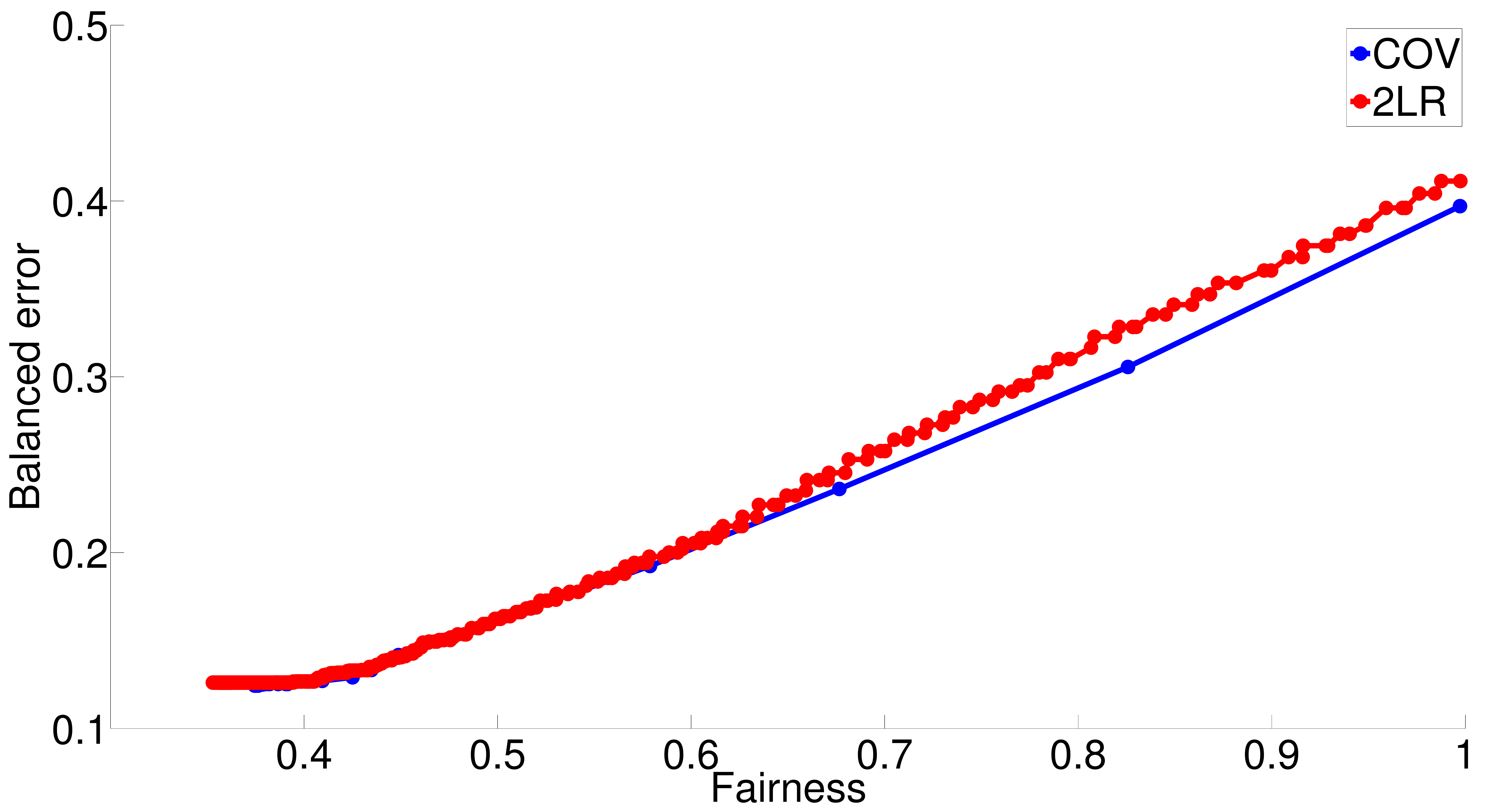} \label{fig:zafar-gaussian}}
	\subfigure[Quadratic features.]{\includegraphics[scale=0.0925]{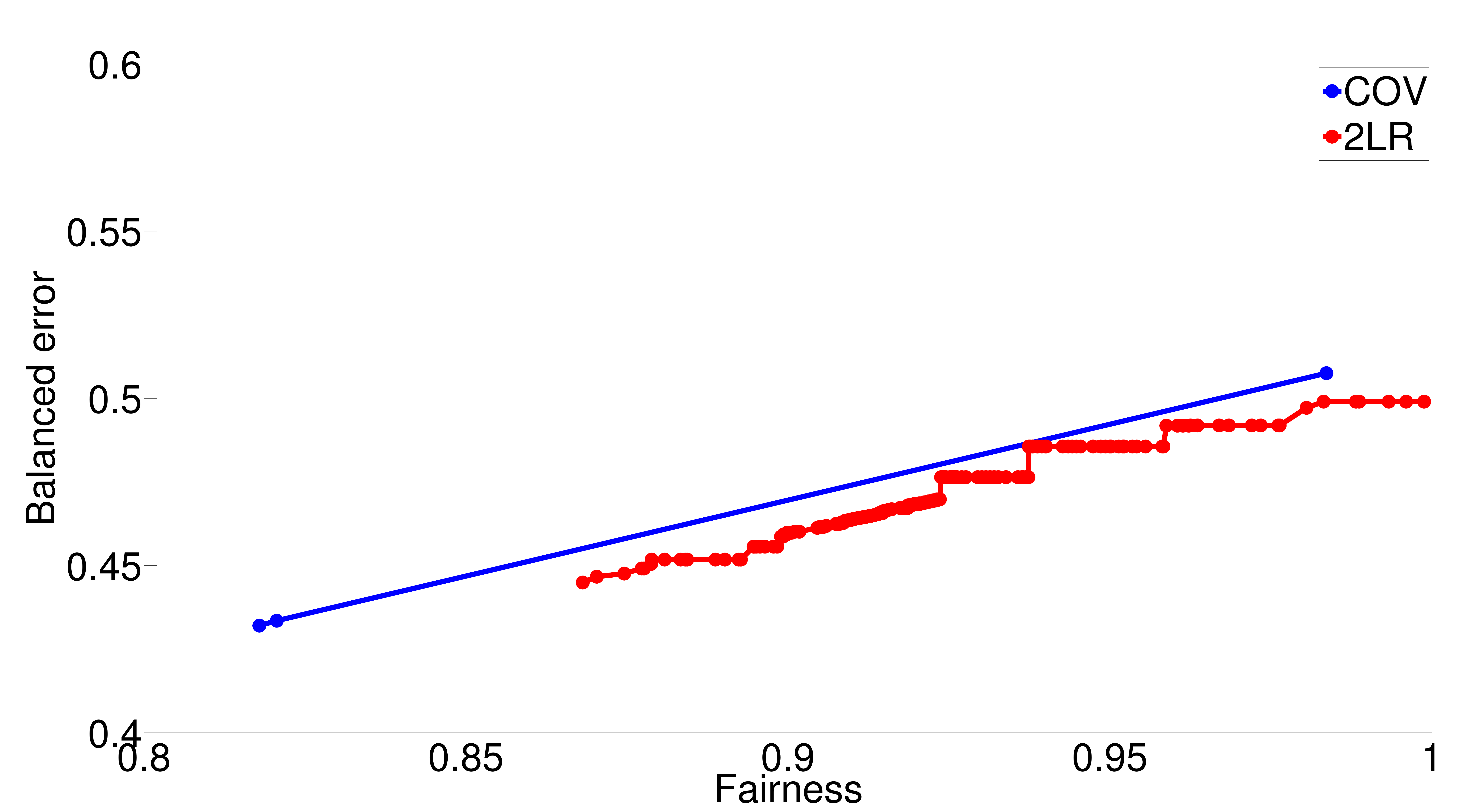} \label{fig:zafar-gaussian-quadratic}}
	\caption{Comparison of plugin ({\tt 2LR}) and {\tt COV} methods as tuning parameters for both are varied, synthetic 2D Gaussian data.}	
\end{figure}

\end{document}